\documentclass[10pt, conference, letterpaper]{IEEEtran}

\usepackage{cite}
\usepackage{amsmath,amssymb,amsfonts}
\usepackage{algorithmic}
\usepackage{graphicx}
\usepackage{textcomp}
\usepackage{xcolor}
\usepackage{amsthm}
\usepackage{subfigure}
\usepackage{graphicx}

\def\BibTeX{{\rm B\kern-.05em{\sc i\kern-.025em b}\kern-.08em
    T\kern-.1667em\lower.7ex\hbox{E}\kern-.125emX}}
 
\newtheorem{myDef}{Definition}
\newtheorem{myConclusion}{Theorem}
\newtheorem{myPosit}{Proposition}

\newtheorem{CheConclusion}{Che's Theorem}

\begin{document}

\title{FMore: An Incentive Scheme of Multi-dimensional Auction for Federated Learning in MEC}
\author{\IEEEauthorblockN{Rongfei Zeng\textsuperscript{\dag}, Shixun Zhang\textsuperscript{\dag}, Jiaqi Wang\textsuperscript{\dag}, and Xiaowen Chu\textsuperscript{\ddag}}
\textsuperscript{\dag}College of Software,  Northeastern University\\
\textsuperscript{\ddag}Department of Computer Science, Hong Kong Baptist University\\
Email: \{zengrf, shixunzhang, jiaqiwang\}@swc.neu.edu.cn, chxw@comp.hkbu.edu.hk
}

\maketitle

\begin{abstract}
Promising federated learning coupled with Mobile Edge Computing (MEC) is considered as one of the most promising solutions to the AI-driven service provision. Plenty of studies focus on federated learning from the performance and security aspects, but they neglect the incentive mechanism. In MEC, edge nodes would not like to voluntarily participate in learning, and they differ in the provision of multi-dimensional resources, both of which might deteriorate the performance of federated learning. Also, lightweight schemes appeal to edge nodes in MEC. These features require the incentive mechanism to be well designed for MEC. In this paper, we present an incentive mechanism FMore with multi-dimensional procurement auction of $K$ winners. Our proposal FMore not only is lightweight and incentive compatible, but also encourages more high-quality edge nodes with low cost to participate in learning and eventually improve the performance of federated learning. We also present theoretical results of Nash equilibrium strategy to edge nodes and employ the expected utility theory to provide guidance to the aggregator. Both extensive simulations and real-world experiments demonstrate that the proposed scheme can effectively reduce the training rounds and drastically improve the model accuracy for challenging AI tasks.
\end{abstract}


\section{Introduction}
Mobile Edge Computing (MEC) \cite{IotJ:Abbas},\cite{Infocom19:Poularakis}, considered as a promising architecture for future networks, enables edge nodes to locally collect and process various data with the remote cloud coordination, which especially appeals to the Internet of Things (IoT), social networking, 5G, etc. In these scenarios, huge amounts of data are generated and further employed by machine learning to provide AI-driven services, such as classification, recommendation, and prediction. However, the proliferation of data will gradually phase-out the traditional paradigm of centrally processing all the data at a remote cloud. Fortunately, edge nodes equipped with powerful computing capability, sufficient Flash storage, etc., accelerate the adoption of local data processing. Recent studies have shown that more than 90\% of data will be stored and processed locally in the near future \cite{JSAC19:Wang}. The salient features of MEC attract not only researchers but also investors from the capital market. An analyst from Goldman Sachs believes that MEC will change the world we live in \cite{GlomanSachs}.

MEC might also boost the widespread use of federated learning. Federated learning \cite{Google17:McMahan}, an emerging division of machine learning, allows collaboratively training a shared model with distributed data, without the need for centralized storage at a cloud. Moreover, federated learning endeavors to address the privacy issue of users who would hesitate to upload their private data to a remote cloud. These two prominent features fascinate the industry. Google applies federated learning to the AI-enabled application Gboard for mobile users \cite{Others18:Hard}, and the open-source framework FATE of federated learning was published by WeBank in April 2019. Some other instances include FedVision, emoji prediction \cite{Others19:Ramaswamy}, anti-money laundering with multiple banks \cite{LINK:FDAI}, etc. Finally, federated learning coupled with MEC is considered as one of the most promising solutions to the AI-driven service provision. 

A plethora of studies concentrate on federated learning \cite{TIST:Yang}, \cite{AAAI19:Yu}, which has already become a hot topic in both academia and industry in recent years. Starting from the impressive work in \cite{Google17:McMahan}, researchers focus on the performance improvement of federated learning \cite{Infocom19:Ou}, \cite{Infocom19:Tran}. They studied the comparison of synchronous and asynchronous aggregations \cite{Others18:Zhao}, the compression of information exchanged in the global aggregation \cite{Others19:Sattler}, \cite{Others19:Wang}, the control algorithm to trade off local updates and global aggregations \cite{Infocom18:Wang}, etc. The security and privacy issue of federated learning is another popular topic \cite{CCS17:Bonawitz}, \cite{Others18:Bagdasaryan}. For example, the chained anomaly detection scheme \cite{Applied18:Preuveneers}, secure global aggregation algorithms \cite{Others19:Bhowmick}, and the privacy-preserving mechanism \cite{Infocom19:Wang} have been proposed in the past two years. In these studies, a critical and optimistic assumption is that voluntary participation of local node is required, without any returns, which does not hold in realistic scenarios of MEC. 

The incentive mechanism is essential and crucial to federated learning in MEC. Since learning operations at edge nodes will consume various resources, such as battery, bandwidth, and computation power, rational edge nodes would not like to get involved in this voluntary collaboration, without any compensation \cite{Ton16:Yang}. Moreover, although federated learning does not need edge nodes to upload their raw data to the remote cloud for the privacy concern, smart malicious attackers may still infer the source information from model parameters \cite{Others19:Kang}. Some potential threats aggravate the reluctance of participation for edge nodes. For service providers, the performance of federated learning is negatively impacted, without sufficient participation of high-quality nodes \cite{ICC19:Nishio}. In sum, the incentive mechanism is indispensable for federated learning. 

Unfortunately, previous incentive mechanisms in other scenarios cannot be directly applied to federated learning in MEC. Most importantly, there exists a widening resource gap between different edge nodes \cite{JSAC19:Wang}, and this gap might deteriorate the performance of federated learning. Consequently, the proposed incentive scheme should encourage more participation of high-quality nodes and choose them to eventually improve the performance of federated learning. Furthermore, resources provided by edge nodes are multi-dimensional and dynamic, and $K$ edge nodes are selected in each game. Besides, the proposed scheme should not introduce much computational cost and communication overhead since these resources are constrained at some nodes. In short, these prominent features should be considered seriously in the design of incentive mechanism in MEC.

In this paper, we study the incentive problem to motivate more high-quality edge nodes with low cost to participate in collaborative learning and eventually improve the performance of federated learning in MEC. To achieve this goal, we borrow and extend the model of multi-dimensional procurement auction proposed by Che in \cite{Economics93:Che}. The aggregator broadcasts bid asks with the selection criteria before participators separately submit bids containing resource qualities and the expected payment. Then, the aggregator chooses $K (K\ge1)$ winners according to the sorted scores. We provide each node with a unique Nash equilibrium strategy to maximize the expected profit, and give guidance to the aggregator to obtain the expected resources, both of which are among the most challenging tasks in the design of incentive scheme. To demonstrate the performance of our proposal, we implement a smart simulator and test with multiple datasets and learning models. We also deploy a real system with 32 nodes.

The main contributions of this paper are three-fold. 
\begin{enumerate}
	\item We present a multi-dimensional incentive framework FMore for federated learning. FMore covers a series of scoring functions and is Pareto efficient for some specific cases. It uses game theory to derive optimal strategies for the edge participators, and leverages the expected utility theory to guide the aggregator to effectively obtain the desired resources.
	\item The proposed scheme is lightweight and Incentive Compatible (IC). The computational overhead and communication costs are negligible in the realistic deployment, and IC indicates that it is useless for edge nodes to declare false qualities in FMore. 
	\item The results of extensive simulations show that FMore is able to speed up federated training via reducing training rounds by  51.3\% on average and improve the model accuracy by 28\% for the tested CNN and LSTM models. Real implementations with 31 edge nodes and one aggregator also witness the improvement of model accuracy by 44.9\% and the reduction of training time by 38.4\%. 
\end{enumerate}

The remainder of this paper is organized as follows. Section II introduces the system model and some preliminaries. In Section III, we present the proposed incentive scheme FMore with the multi-dimensional auction, followed by some theoretical results in Section IV. Extensive performance evaluations are presented in Section V. Section VI surveys related work, and Section VII concludes the whole paper. 

\section{System Model and Preliminaries}
\subsection{System Model}
In this paper, we consider a typical MEC network, where edge nodes such as micro servers, home gateways, laptops, and sensors, are connected to a remote cloud. The number of potential edge nodes is large, and various resources of each node are dynamic. Edge nodes have constrained resources for federated learning since they have other important tasks. In Fig. \ref{MCEsystem}, an aggregator exists in the remote cloud to orchestrate federated learning with distributed edge nodes. To obtain a well-trained global model, the aggregator has motivations to pay for the recruited edge nodes. Edge nodes demonstrate no intention to upload their private data to the remote cloud, and would not like to offer their dynamic resources to the aggregator unless they are paid for their contributions. Both the aggregator and edge nodes are assumed to obey the contracts they negotiate, and edge nodes are also assumed to be trustable that they will provide what they bid. Many techniques such as blacklist can be applied to the defaulter. Similar to \cite{Economics93:Che}, \cite{TMC18:Lin}, we also adopt the independent private value model for edge nodes (sellers) and aggregator (buyer). Finally, some threats, e.g. collusion attacks and false name attacks, are not considered in this paper. In Table \ref{Notations}, we summarize some notations frequently used through this paper. 

\begin{figure}[tp]
\centering
\includegraphics[scale = 0.5, width = 8.8cm]{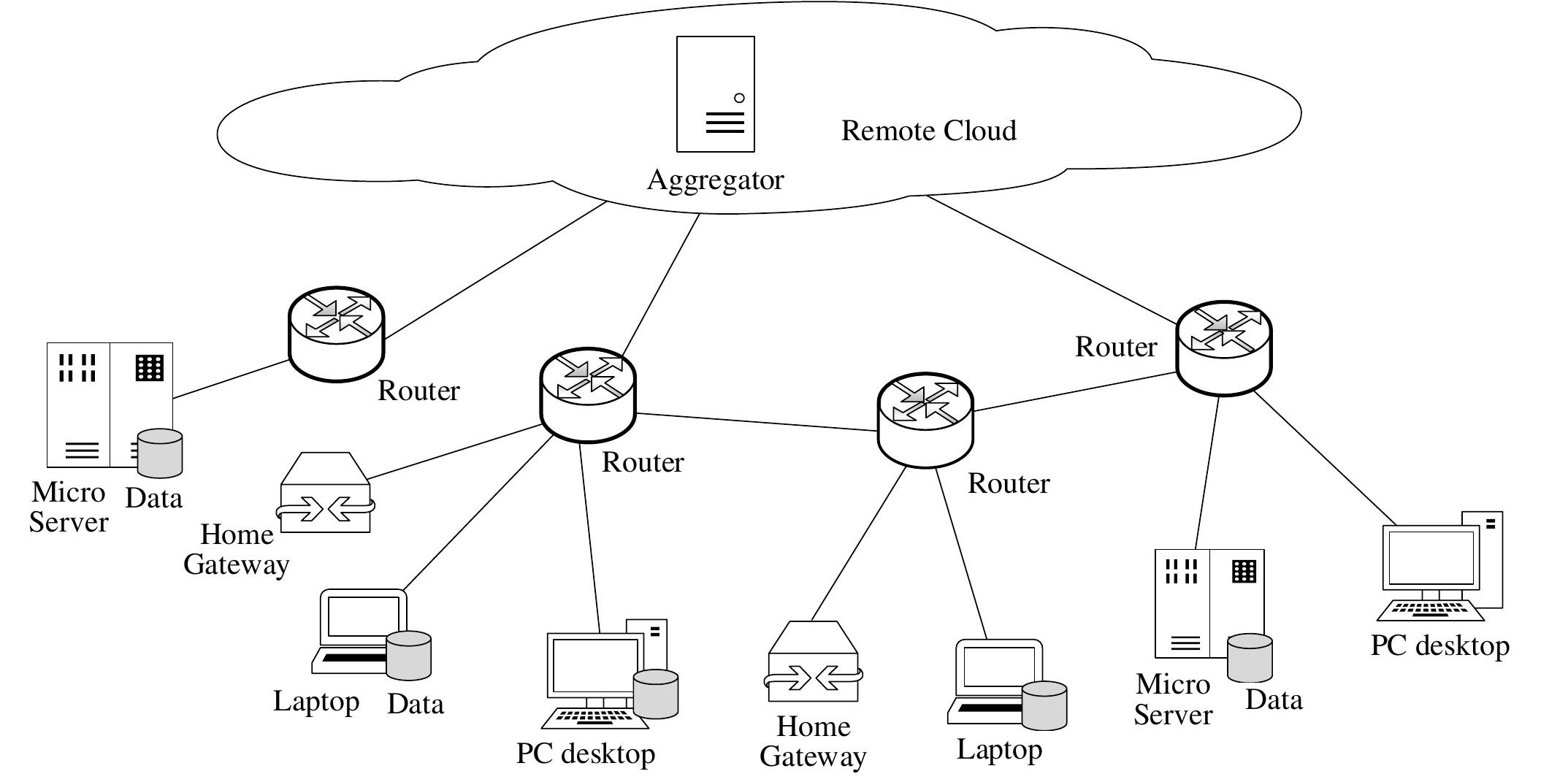}
\vspace{-0.1cm}
\caption{The system model of mobile edge computing} 
\label{MCEsystem}
\vspace{-0.2cm}
\end{figure}

\subsection{Preliminaries}
Federated learning is designed to train a shared global model that minimizes the global loss function $F(w)$ in a cooperative and distributed manner \cite{Google17:McMahan}. Formally, the goal of federated learning is to find model parameters $w^*$, which satisfy \begin{equation}
w^*=arg \ min \  F(w).
\end{equation}
Typically, the training process takes a number of rounds to converge. In each round, the aggregator randomly chooses $K$ nodes from all the $N$ edge nodes, and then distributes the global parameter $w(t)$, where $t = 0,1,...,{T-1}$ is the iteration index, to those selected nodes. Based on the global parameter $w(t)$, the chosen node trains the shared model with its local data, i.e., 
\begin{equation}
w_i(t+1)=w(t)-\eta \nabla F_i(w_i(t)),
\label{Eq:localtraining}
\end{equation}
where the parameter $\eta$ is the step size. After the local training, these nodes upload their model parameters to the aggregator, and the aggregator generates global parameters of $t+1$ as 
\begin{equation}
w(t+1)=\frac{\sum_{i=1} ^ N D_i w_i(t+1)}{\sum_{i=1}^N D_i},
\label{Eq:globalaggregation}
\end{equation}
where $D_i$ is the data size of node $i$. Then, the aggregator will initialize the next round of training by randomly choosing $K$ nodes. When the accuracy of global model satisfies the requirement or the training time exceeds the predefined threshold, this training process terminates. Briefly, federated learning consists of many iterations of global aggregation and local training in Fig. \ref{Procedure}(a). Also, model accuracy and training rounds are two critical performance metrics. Finally, we should mention that our proposed scheme can be applied to this classic federated learning \cite{Google17:McMahan} as well as other paradigms \cite{Others18:Kim}. 

\begin{table}[tp]
\caption{The notations frequently used in this paper}
\vspace{-0.15cm}
\label{Notations}
\centering
\begin{tabular}{ll}
\hline
 Notation & Description \\
\hline 
$N, K$ & total number of edge nodes and size of winner set\\
$m$ & number of resource types \\
$\mathbb{W}, \mathbb{N}$ & winner set and edge node set\\
$q_{ij}$ & quality of $j$th resource of user $i$\\
$\mathbf{q_i}$ & quality vector of user $i$\\
$\hat{q}_i$ &  declared quality of $i$th resource\\ 
$p_i$ & payment of user $i$\\
$\theta_i$ & private cost parameter of user $i$\\
$F(\theta_i)$& cumulative distribution function of $\theta_i$\\
$f(\theta_i)$ & probability density function of $\theta_i$\\
$S(\cdot)$ & scoring function given by the aggregator\\
$c(\cdot)$ & cost function \\
$\pi_i(\cdot), V(\cdot)$ & profit functions of user $i$ and the aggregator \\
$U(\cdot)$ & utility function of the aggregator\\
$t_i^{ne}$ & Nash equilibrium strategy of node $i$\\
$t_{-i}^{ne}$ & Nash equilibrium strategies except node $i$ \\
$\psi$ & probability of edge node being selected \\
$[K]$ & auction with $K$ winners\\
\hline
\end{tabular}
\vspace{-0.1cm}
\end{table}

\vspace{0.1cm}
\section{FMore: The Proposed Incentive Scheme}

In this section, we present the incentive mechanism FMore based on the multi-dimensional procurement auction and detail the design rationale for each step in FMore. To explicitly illustrate our proposal, we also describe a walk-through example with five edge nodes. Further discussions are provided for specific scenarios as well. 

\begin{figure}[tp]
\centering
\includegraphics[scale = 0.4, width = 9cm, height = 4.8cm]{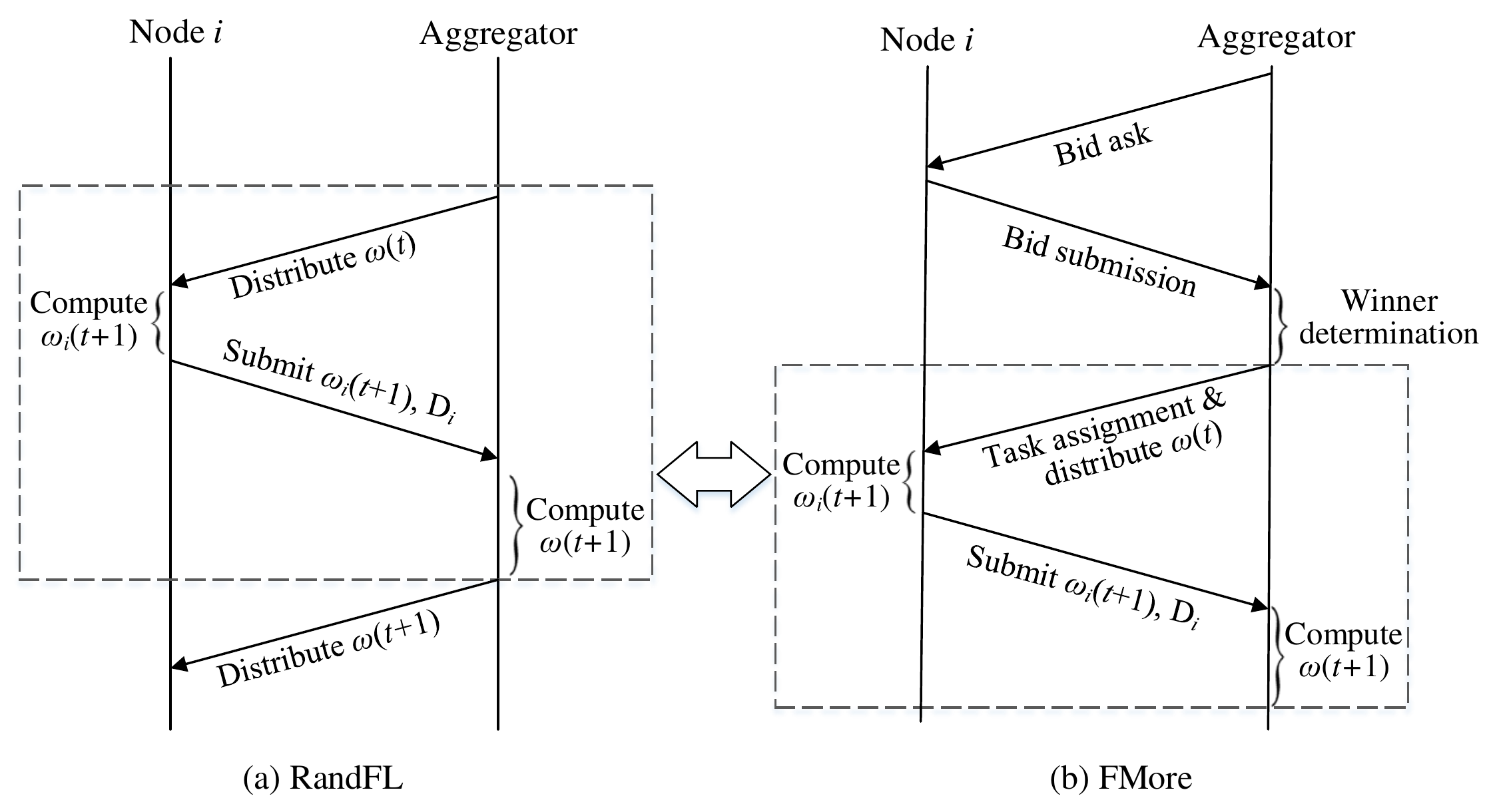}
\vspace{-0.5cm}
\caption{The procedure of RandFL and FMore} 
\label{Procedure}
\vspace{-0.1cm}
\end{figure}

\subsection{The Description of FMore} 
The proposed incentive framework FMore consists of six steps, i.e., bid ask, bid collection, winner determination, task assignment, local training, and global aggregation, in each round of training, as shown in Fig. \ref{Procedure}(b). The latter three steps are similar to the classic federated learning in \cite{Google17:McMahan} (referred to as RandFL). The computational cost and communication overhead are only introduced in the former three steps, which should be considered seriously in the design of FMore. 

\textbf{(1) Bid Ask:}
In each round of federated learning, the aggregator initially broadcasts a scoring rule $S(q_1, \cdots, q_m, p)$, where $\mathbf{q} = (q_1, \cdots, q_m)$ is the quality vector of resources, and $p$ is the expected payment that the edge node bids with the provision of $\mathbf{q}$. The resources considered in FMore include local data, computation capability, bandwidth, CPU cycle, etc. In addition, the aggregator leverages the scoring function to choose participators. We formulate $S(\mathbf{q_i},p_i)$ as a quasi-linear function
\begin{equation}
S(q_{i1}, q_{i2}, \cdots, q_{im}, p_i) = s(q_{i1}, q_{i2}, \cdots, q_{im}) - p_i,
\end{equation}
where subscript $i = 1, 2, \cdots, N$ is the node index. Comparing with the size of model parameters, we can neglect communication overhead in this step. This is because only a score function and simple requirements are delivered from the aggregator to edge nodes, and the corresponding data size is just a few bytes.

Many scoring functions can be included in FMore. For instance,  $s(\cdot)$ can be set as the utility function $U(q_1, \cdots, q_m)$ of the aggregator. Some classic utility functions include the perfect substitution utility function, the perfect complementary function, and the general Cobb-Douglas function, which are separately denoted as 
\begin{eqnarray*}
&&s(\cdot) = \alpha_1 q_1 + \cdots + \alpha_m q_m, \\
&&s(\cdot) = min \{\alpha_1 q_1, \cdots, \alpha_m q_m\},\\
&&s(\cdot) = q_1^{\alpha_1}q_2^{\alpha_2} \cdots q_m^{\alpha_m},
\end{eqnarray*}
where $\alpha_1, \cdots, \alpha_m$ are coefficients. We may add the constraint $\sum \alpha_i = 1$, but it is not imperative. In the above functions, the additive form is preferred to perfect substitution resources such as GPU and CPU, while the perfect complementary form might be the best choice for scenarios where both bandwidth and computing power are considered simultaneously. 

\textbf{(2) Bid Collection:}
When edge nodes receive a bid ask with scoring function $S(\cdot)$, they separately base on their available resources to decide whether to bid or not. According to the private value model in \cite{Ton16:Yang}, edge node $i$ has a private cost parameter $\theta_i$, and then it can get the private cost function $c(q_1, \cdots, q_m, \theta_i)$. Note that the cost function $c(\cdot)$ is an increasing function of $q_i$. In this paper, we assume single crossing conditions $c_{qq} \ge 0$, $c_{q\theta} >0$, and $c_{qq\theta} \ge 0$, which mean the marginal cost increases with the parameter $\theta$. Before bidding, each node learns its private cost parameter $\theta$ and gets the Cumulative Distribution Function (CDF) $F(\theta)$ from the historical data. It is assumed that $\theta_i$ is independently and identically distributed over the range of $[\underline{\theta}, \overline{\theta}] \ (0<\underline{\theta} < \overline{\theta} <\infty)$. There also exists a positive and continuously differentiable density function $f(\theta)$.  

How much to bid? As a rational edge node, node $i$ needs to choose $\mathbf{q_i}$ and $p_i$ to maximize the following profit function
\begin{equation}
\pi_i(q_1, \cdots, q_m,p_i) = p_i  -c(q_1, \cdots, q_m, \theta_i).
\end{equation}
In this optimization problem, one of the constraints is Individual Rationality (IR), which implies that any node will not participate in federated learning when its profit is negative. In other words, $\pi_i(\mathbf{q_i},p_i) \ge 0$. Let $\pi_i(\mathbf{q_i},p_i) = 0$ denote that node $i$ will not join in the training. In the next section, we will present the theoretical results of optimal strategy for each edge node to compete with others. 

When edge node $i$ submits its bid $(\mathbf{q_i},p_i)$ to the aggregator, the technique of sealed-bid auction is adopted, indicating that this bid is only known to the aggregator and node $i$. The sealed-bid auction is quite suitable for network scenarios and can be easily implemented by FMore. 

\textbf{(3) Winner Determination:} 
When the aggregator collects sufficient bids or the timer with a predefined threshold expires, the aggregator finishes the bid collection process. Then, it starts to determine the winners. In this paper, we extend the classic multi-dimensional auction to multiple winners. In the winner determination, the aggregator has to maximize the profit function $V(\cdot)$ as 
\begin{equation}
V = \sum_{ i \in \mathbb{W}} (U(q_{i1}, q_{i2}, \cdots, q_{im}) - p_i),
\vspace{-0.1cm}
\end{equation}
where $\mathbb{W}$ is the winner set, and $U(\cdot)$ is the utility function of $\mathbf{q}=(q_1, q_2, \cdots, q_m)$. Similar to the literature \cite{Economics93:Che}, we also assume that $U'(\cdot) \ge 0$, $U''(\cdot)< 0 $, $\lim_{q = 0} U'(q) = \infty $, and $\lim_{q = \infty} U''(q) = 0$. Moreover, the constraint of IR, i.e., $V \ge 0$, should be satisfied for the rational aggregator as well. 

In FMore, the aggregator chooses $K$ edge nodes with the best scores to construct the winner set $\mathbb{W}$. The parameter $K$ is decided by the aggregator and can be estimated with historical data. Besides the winner determination, the aggregator has to perform the payment allocation. Both the first-price auction and the second-price auction can be applied to FMore. We use the first-price auction for simplicity in this paper. 

\textbf{(4) Task Assignment, Local Training and Global Aggregation:}
The last three steps are similar to the classic federated learning RandFL, where winners locally train the model with declared resources, according to Eq. (\ref{Eq:localtraining}). After finishing local updates, they submit the result of model parameters to the aggregator and then obtain the corresponding payment. If any edge node does not comply with the contract, it will be put into the blacklist by the aggregator.  

The pseudocode of FMore is given in Algorithm 1. Compared with RandFL, our scheme FMore just adds one round of information exchange between edge nodes and the aggregator, and the total communication cost is a linear function of $N$. The computational cost contains the calculation of optimal strategy at each edge node and the sorting operation at the aggregator. The time complexity of optimal strategy computation is linear, which can be found in Section IV. Thus, our proposed scheme is lightweight, which is much more appropriate for MEC.  

\begin{figure}[!tp]
\centering
\includegraphics[scale = 0.8, width = 8.9cm]{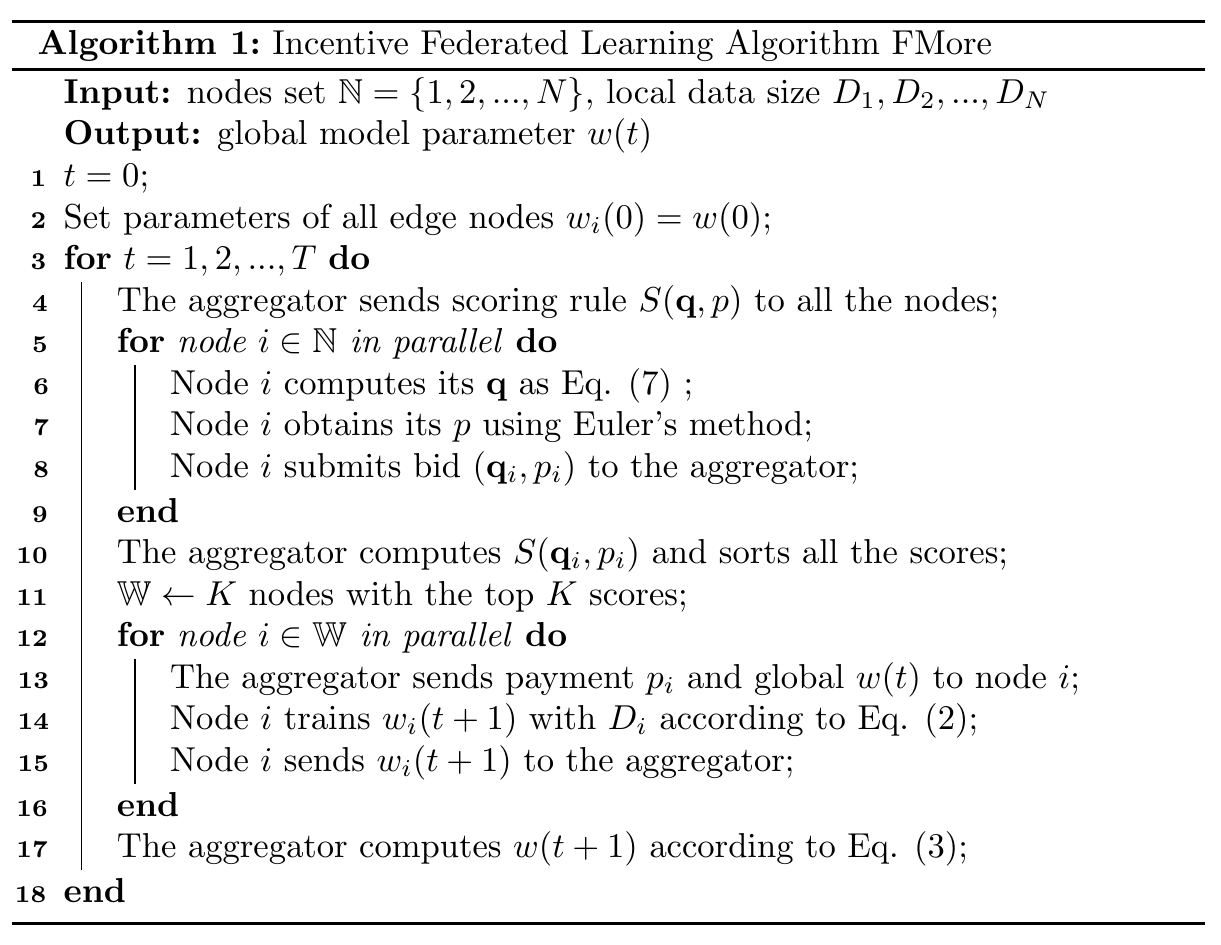}
\vspace{-0.5cm}
\end{figure}

\subsection{A Walk-Through Example}
In Fig. \ref{Example}, we present a walk-through example with five edge nodes ${\mathbb{N}=\{A, B, C, D, E\}}$ and consider two types of resources, i.e., training data and bandwidth. The data size and bandwidth are separately over the range of $[1000, 5000]$ and $[5Mb, 100Mb]$. The public scoring function is set as $S(\mathbf{q},p)=min \{\alpha_1 q_1, \alpha_2 q_2\}-p$, where coefficients $\alpha_1$ and $\alpha_2$ are to balance different types of resources and both set to 0.5. In addition, $q_1, q_2$, and $p$ are normalized by the technique of min-max normalization to compute the scores for simplicity. It should be noted that the strategy for each node might not be optimal in this example, and we will provide the Nash equilibrium strategy to a rational node in Section IV. 

In the first round of training, these five edge nodes individually submit the bids $(q_1, q_2, p)$ as $(4000, 85Mb, 0.20)$, $(3000, 35Mb, 0.10)$, $(3500, 75Mb, 0.18)$, $(5000, 85Mb, 0.20)$, and $(5000, 100Mb, 0.20)$. After collecting all the bids, the aggregator computes the scores, sorts them in the descending order, and chooses three winners ($K=3$ and $\mathbb{W} = \{A, D, E\}$) with top three scores. The payments for winners are 0.175, 0.221, and 0.300 in the first-price auction. The aggregator distributes the global parameters to these three winners for their local learnings. When winners finish local trainings, the aggregator performs the global aggregation as Eq. (\ref{Eq:globalaggregation}) and then terminates this round of training. 

In the second round, these nodes might change their bids as $(4000, 85Mb, 0.16)$, $(3500, 45Mb, 0.1)$, $(4000, 80Mb, 0.15)$, $(4000, 80Mb, 0.2)$, and $(5000, 100Mb, 0.3)$. We take node $C$ as an example to illustrate the dynamic provision of resources. The reasons why node $C$ changes its bid include but not limit to: (1) the available resources are changed; (2) the private cost parameter $\theta$ is reestimated and revised; and (3) node $C$ trades revenue for others such as reputation. Any unknown reason leads to that node $C$ submits the bid as $(4000, 80Mb, 0.15)$ and it is ranked as the first this time. The selection of winners is similar to the last round, and the set $\mathbb{W} = \{A, C, E\}$ is constructed. These three nodes are responsible for local training in this round, and payments are 0.16, 0.15, and 0.3 in the first-price auction. Similarly, we perform the processes of bid ask, bid collection, winner determination, task assignment, local training, and global aggregation iteratively until the model accuracy satisfies the requirement.  

\begin{figure}[tp]
\centering
\includegraphics[scale = 0.5, width = 8.6cm]{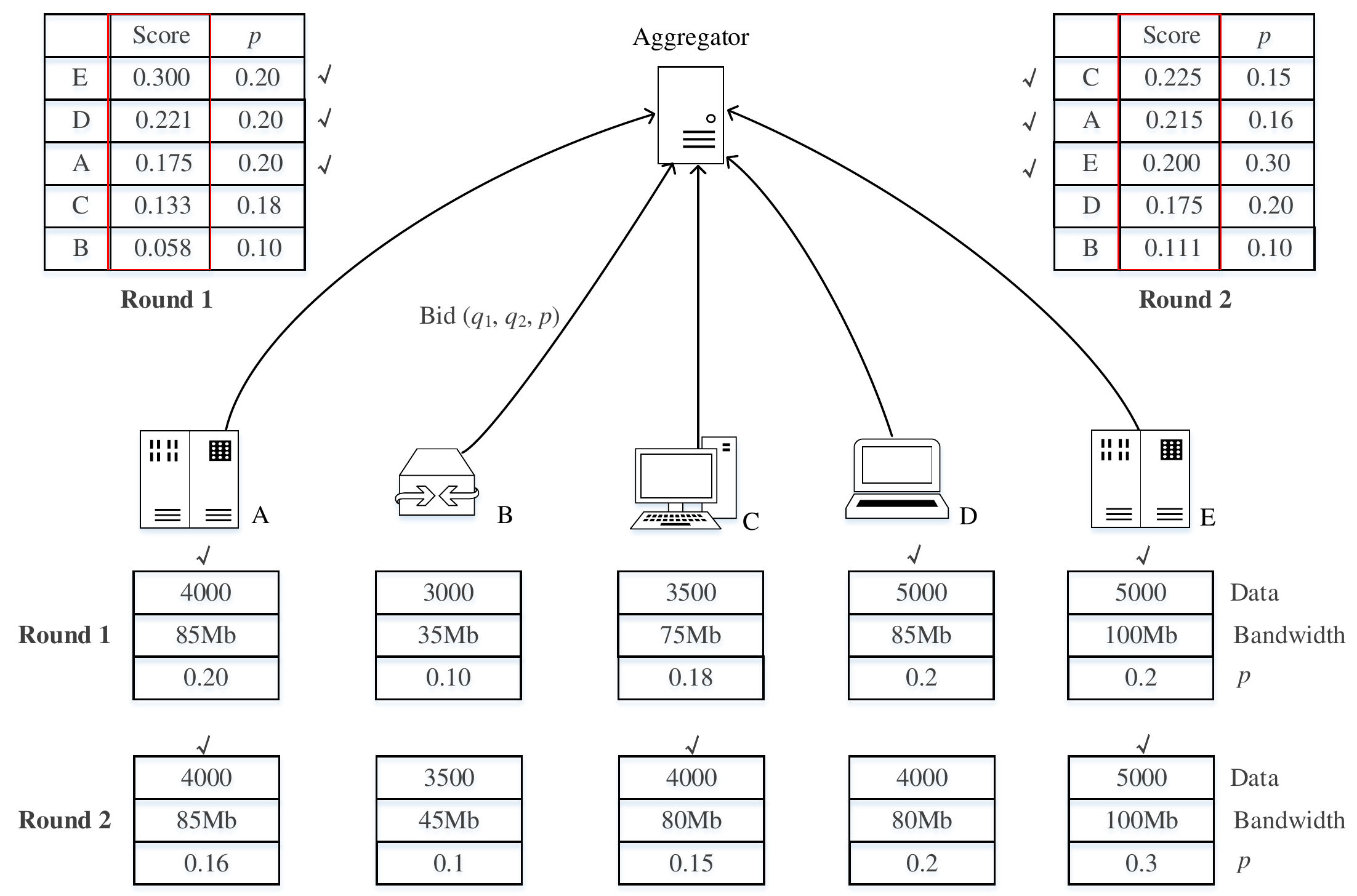}
\caption{The example of FMore with five nodes} 
\label{Example}
\vspace{-0.1cm}
\end{figure}

\subsection{Discussion}
In MEC, the resource provision is dynamic and distinct, and there exist some nodes that have sufficient local data with high quality. However, the situation is changed in some other scenarios where the resources of participators are relatively stable and the local data size is tremendously small for each node. In such catastrophic cases, selecting fixed nodes with limited data and inferior-quality resources may negatively affect the performance of federated learning. For instance, the overfitting problem is frequently encountered in such scenarios. To tackle these problems, we extend FMore for widespread application.

In the winner determination phase of FMore, $K$ top-score nodes are definitely added to the winner set. Now, we revise this step as follows: nodes in the descending order of scores will be individually added to the winner set $\mathbb{W}$ with probability $\psi$ until $K$ nodes are chosen in the set $\mathbb{W}$. This can be achieved by changing Line 11 of Algorithm 1. We name this extension as $\psi$-FMore, and FMore is a special case of $\psi$-FMore with $\psi =1$. In $\psi$-FMore, we can construct the winner set $\mathbb{W}$ of $K$ nodes with probability $Pr(\psi) = \sum_{i=0}^{N-K} C_{i+K}^i (1-\psi)^i\psi^K$. It can be easily verified that the probability $Pr(\psi)$ approaches to one with many appropriate parameters. In addition, the parameter $\psi$ should be carefully set to balance the model accuracy and the training speed in extreme scenarios, since small $\psi$ might deteriorate FMore into the classic federated learning RandFL. In Section V, we demonstrate the impacts of parameter $\psi$ on the performance of federated learning. 

\section{Optimal Strategy and Utility Analysis}
In this section, we first analyze Nash equilibrium strategies for edge nodes and present closed-form theoretical results. The impacts of parameters $N$ and $K$ are also studied here. Then, we provide guidance to the aggregator to get the expected resources. Finally, we also prove that FMore is Pareto efficient and IC.

In FMore, it is the key to discover the Nash equilibrium strategy for an edge node. The Nash equilibrium strategy $t_i^{ne}$ indicates that it is the optimal choice for node $i$, no matter what strategies other players choose. We present the definition of Nash Equilibrium in our auction as follows:
\begin{myDef}
Nash Equilibrium: The strategy set of all the participators $(t_1^{ne}, t_2^{ne},\cdots, t_N^{ne})$ is a Nash Equilibrium, if any edge node $i$ has 
\begin{equation*}
\pi_i (t_i^{ne}, t_{-i}^{ne}) \ge \pi_i (t_i, t_{-i}^{ne}), t_i \ge 0,
\end{equation*} 
where $t_{-i}^{ne} = (t_1^{ne}, \cdots, t_{i-1}^{ne}, t_{i+1}^{ne}, \cdots, t_N^{ne}) $, and $t_i \in S_i$ is any of the strategies for node $i$. 
\end{myDef}

The Nash equilibrium strategy $t_i^{ne}$ for edge node $i$ consists of two components, i.e., the qualities of resources and the expected payment. For the former part, Che's Theorem 1 has already found that the choice of quality is only relevant to the private cost parameter $\theta$. In other words, Che's Theorem 1 shows that quality can be independently chosen. Then, we provide the unique Nash equilibrium strategy in the one-winner game in Che's Theorem 2 and extend the theoretical results to the two-winners game in Proposition 1. Since the proofs of Che's Theorem 1 and Che's Theorem 2 can be referred to in \cite{Economics93:Che}, we omit here for space limitation. We also omit the subscript $i$ in the following analysis for simplicity. 

\begin{CheConclusion} 
In the first-price auction with $K (K \ge 1)$ winners, the quality of resource is chosen at $q_s(\theta)$ for all $\theta \in [\underline{\theta}, \overline{\theta}]$, where $q_s(\theta) = arg \ max \ s(q)-c(q, \theta)$.
\end{CheConclusion}

\begin{CheConclusion} 
The unique Nash equilibrium strategy $t^{ne}(\theta)=(q_s(\theta), p_s(\theta))$ for each node in the first-price auction with one winner is given as 
\begin{eqnarray*}
&&q_s(\theta) = arg \ max \ s(q)-c(q,\theta) \\
&&p_s(\theta) = c(q_s,\theta) +  \int_{\theta}^{\overline{\theta}} c_{\theta}(q_s(t), t) \left( \frac{1-F(t)}{1-F(\theta)} \right) ^{N-1} dt
\end{eqnarray*} 
\end{CheConclusion}

\begin{myPosit} 
The unique Nash equilibrium for each node in the first-price auction with two winners can be denoted as 
\begin{eqnarray*}
&&q_s(\theta) = arg \ max \ s(q)-c(q,\theta) \\
&&p_s(\theta) = c(q_s,\theta) +  \int_{\theta}^{\overline{\theta}} c_{\theta}(q_s(t), t) \left( \frac{1-F(t)}{1-F(\theta)} \right) ^{N-2} dt
\end{eqnarray*} 
\end{myPosit}

The proof of Proposition 1 is similar to Che's Theorem 2. The only difference is that the probability $Pr\{win | S(q(\theta),p)\}$ is computed as the sum of winning probability with the first score and winning probability with the second score. Interested readers can refer to \cite{Economics93:Che} for details.

\begin{myConclusion} 
The unique Nash equilibrium of each node in the first-price auction with $K$ winners can be denoted as 
\begin{eqnarray}
q_s(\theta) &= & arg \ max \ s(q)-c(q,\theta), \\
p_s(\theta) &= &c(q_s, \theta) + \int_0^{u} \left( \frac{g(x)}{g(u)} \right) dx, \\
g(u) &=& \sum_{i=1}^M [1-H(u)]^{i-1}[H(u)]^{N-i}, \\
u(\theta) &=& s(q(\theta))-c(q(\theta), \theta).
\end{eqnarray} 
\end{myConclusion}
\begin{proof}
The expected profit $\pi(\cdot)$ of edge node is denoted as 
\begin{equation}
\pi (q, p | \theta) = (p-c(q(\theta), \theta)) Pr\{win | S(q(\theta),p)\}.
\label{GoalPi}
\end{equation}
We define the maximum score $u = X(\theta)$ and $u_0$ as 
\begin{eqnarray*}
u &=& s(q(\theta))-c(q(\theta), \theta), \\
u_0 &=& s(q_s(\theta))-c(q_s(\theta), \theta).
\end{eqnarray*}
Since the CDF of $\theta$ is $F(\theta)$, we can use the Envelope theorem to get the CDF of $X(\theta)$ as $H(x) = 1-F(X^{-1}(x))$. We also define $b(u) = b(X(\theta)) = S(q(\theta),p)$. Then,  the expected profit $\pi (q, p |\theta)$ can be represented by 
\begin{eqnarray*}
\pi (q, p | \theta) &=& (u_0-b(u)) g(u),\\
g(u) &=& g(\theta) = Pr\{win | X(\theta)\} \\ 
&=& \sum_{i=1}^K [1-H(X(\theta))]^{i-1}[H(X(\theta))]^{N-i} \\
&=& \sum_{i=1}^K [1-H(u)]^{i-1}[H(u)]^{N-i}
\end{eqnarray*}
To get the maximum $\pi(q, p | \theta)$, we have
\begin{equation*}
\frac{\partial \pi}{\partial u} = -b'(u)g(u) +(u_0-b(u))g'(u), 
\end{equation*}
and $\frac{\partial \pi}{\partial u} | _{u =u_0} =0$. Define $\varphi(u) = \frac{g'(u)}{g(u)}$. Then, we can easily get the first order linear differential equation 
\begin{eqnarray}
b'(u_0) + \varphi(u_0) b(u_0) = u_0 \varphi(u_0). 
\end{eqnarray}
We can solve this equation to get a unique $b(u)$ with the initial condition $b(0) = 0$. Then, we can get 
\begin{equation*}
b(u) = u - \int_0^{u} \left( \frac{g(x)}{g(u)} \right) dx. 
\end{equation*}
Since the quality $p_s(\theta) - c(q_s, \theta) = u - b(u)$ holds at the equilibrium point, we can get the equilibrium $p_s(\theta)$ as 
\begin{equation*}
p_s(\theta) = c(q_s,\theta) + \int_0^{u} \left( \frac{g(x)}{g(u)} \right) dx.
\end{equation*}
Thus, we obtain the conclusion. 
\end{proof}

It should be noted that the closed-form of $p_s ({\theta})$ is extremely complicated. We can use classic numerical methods, e.g., the Euler method and the Runge-Kutte method, to get the result of $p_s(\theta)$. Here, the Euler method can be described as 
\begin{eqnarray}
\frac{dy}{dx}&=&f(x,y),\\
y_{n+1} &=& y_n + f(x_n, y_n)h,
\end{eqnarray}
where $h$ is the step size. Eq. (12) can be represented like Eq. (13). Then, we can get $p_s(\theta)$ with the complexity of linear time. 

\begin{myConclusion} 
In a game with $K$ winners, the expected profit of edge node $\pi (q,p | \theta)$ is a decreasing function of the total node number $N$. 
\end{myConclusion}
\begin{proof}
From the proof of Theorem 1, we can easily have 
\begin{eqnarray*}
\frac{\partial \pi}{\partial N} = \pi(q,p|\theta) ln H(X(\theta)) \le 0.
\end{eqnarray*}
Since CDF $H(\cdot) \le 1$, we can have that the profit function $\pi(\cdot)$ is a decreasing function of $N$, when $\theta \ne \overline{\theta}$.
\end{proof} 

Theorem 2 conforms to the fact that when more edge nodes join in the game, the competition becomes more severe, and the profit for each node is correspondingly reduced. Hence, the increase of $N$ will benefit the aggregator, which is the reason why the incentive mechanism is significant. Next, we will demonstrate that the profit of the participator is increased by the design of multiple winners as well.

\begin{myConclusion} 
In a game with $N$ nodes, the expected profit of each edge node $\pi (q,p | \theta)$ is an increasing function of winner number $K$. 
\end{myConclusion}
\begin{proof}
The equilibrium strategies for $K$th and $(K+1)$th winner are separately denoted as $(p_{[K]}^{ne},q_{[K]}^{ne})$ and $(p_{[K+1]}^{ne},q_{[K+1]}^{ne})$. We also use $\pi_{[K+1]} (\cdot)$ to denote the profit in the game having $(K+1)$ winners. Then, we have 
\begin{eqnarray*}
&&\pi_{[K+1]} (q_{[K]}^{ne}, p_{[K]}^{ne} | \theta) - \pi_{[K]} (q_{[K]}^{ne}, p_{[K]}^{ne} | \theta) \\
&&= (p_{[K]}^{ne} - c(q_{[K]}^{ne},  \theta)) (1-H(X(\theta)))^K (H(X(\theta)))^{N-K-1}
\end{eqnarray*}
It can be seen that
\begin{equation*}
\pi_{[K+1]} (q_{[K]}^{ne}, p_{[K]}^{ne} | \theta) - \pi_{[K]} (q_{[K]}^{ne}, p_{[K]}^{ne} | \theta) >0.
\end{equation*}
Since the strategy $(p_{[K]}^{ne},q_{[K]}^{ne})$ might not be the Nash equilibrium strategy for the game with $(K+1)$ winners, we have 
\begin{equation*}
\pi_{[K+1]} (q_{[K+1]}^{ne}, p_{[K+1]}^{ne} | \theta) \ge \pi_{[K+1]} (q_{[K]}^{ne}, p_{[K]}^{ne} | \theta)
\end{equation*}
Thus, we can get that $\pi_{[K+1]} (\cdot) \ge \pi_{[K]} (\cdot)$ and $\pi (\cdot)$ is an increasing  function of $K$.
\end{proof}

\begin{myPosit} 
Suppose that all the participators have the same private value $\theta$ and we must select $K$ winners from $N$ nodes, then adding probability $\psi$ to each node will not impact its winning probability.
\end{myPosit} 

It is an ideal model assumption in Proposition 2, the proof of which is given in Appendix A. In realistic scenarios, the private value $\theta$ is not identical for most nodes. For the node selected by FMore with high probability, $\psi$-FMore will negatively impact its winning probability. On the contrary, the winning probability of a low-score node will be improved by our $\psi$-FMore. More nodes are involved by $\psi$-FMore, and the critical parameter $\psi$ should be carefully chosen. In sum,  $\psi$-FMore improves the performance of federated learning due to the increased data diversity in extreme cases. 

\begin{myPosit} 
In a game with multi-dimensional resources and $K$ winners, the choice of $(q_1, \cdots, q_m)$ is independent with $p$. The quality can be separately computed according to Che's Theorem 1. 
\end{myPosit}

For multi-dimensional resources, the quality combination is computed by maximizing $s(q_1, \cdots, q_m) - c(q_1, \cdots, q_m, \theta)$, and the result contains a set of quality combinations. The proof of Proposition 3 is presented in Appendix B. For the aggregator, it can set the weight of  $q_i$ in the function $s(\cdot)$ to get what it needs. In the following proposition, we will provide guidance to the aggregator to get the expected resources in an efficient market. 

\begin{myPosit} 
When we consider the general Cobb-Douglas utility function $s(\cdot) = \prod_{i=1}^m q_i^{\alpha_i}$ and the additive cost function $c(\cdot)= \theta( \sum_{i=1}^m\beta_i q_i)$, where $\sum_{i=1}^m \alpha_i = 1$ and $\sum_{i=1}^m \beta_i = 1$, the aggregator can adjust the parameters $(\alpha_1, \cdots, \alpha_m)$ to get different proportion of resources. That is
\begin{equation*}
\frac{q_i^*}{q_j^*} = \frac{\alpha_i}{\alpha_j} \cdot \frac{\tilde{\beta}_j}{\tilde{\beta}_i},
\end{equation*}
where $\tilde{\beta}_i$ is the estimation of cost coefficient for $q_i$ according to the historical data in the public and efficient market. 
\end{myPosit}
  
Proposition 4 can be proved with the expected utility theory that the general Cobb-Douglas utility function is maximized with the cost constraint. The proof is given in Appendix C. In this way, the aggregator is able to get the expected resources from a macro view.  
\begin{myConclusion} 
When the utility function $U(\cdot)$ of the aggregator is equal to $s(\cdot)$ and has the additive form, our proposed FMore is Pareto efficient. 
\end{myConclusion} 
\begin{proof}
Pareto efficiency is equivalent to that the social surplus is maximized. The social surplus $SS$ is given as  
\begin{eqnarray*}
SS&=&U(q_1, \cdots, q_m) -  \sum_{i \in \mathbb{W}} c(q_{i1}, \cdots, q_{im}, \theta_i) \\
&=& \sum_{i \in \mathbb{W}} \big( s(q_{i1}, \cdots, q_{im}) - c(q_{i1}, \cdots, q_{im}, \theta_i) \big)
\end{eqnarray*}
Since the quality of each winner is chosen as $(q_1, \cdots, q_m) =  arg \ max \ s(q_1, \cdots, q_m) - c(q_1, \cdots, q_m, \theta)$, we can directly arrive at the conclusion. 
\end{proof}

\begin{myConclusion} 
FMore is Incentive Compatible (IC).
\end{myConclusion}
\begin{proof}
The payment $p$ is computed by maximizing the expected profit in Eq. (\ref{GoalPi}) with the combination of quality $(q_1, \cdots, q_m)$, and the corresponding score is $S_0 = s(q_1, \cdots, q_m) - p$. If a node declares the malicious quality $\{(\hat{q}_1, \cdots, \hat{q}_m) |\exists j, \hat{q}_j<q_j \}$ and payment $p$, we can have
\begin{eqnarray*}
&&s(q_1, \cdots, q_m) - c(q_1, \cdots, q_m, \theta) > \\
&&s(\hat{q}_1, \cdots, \hat{q}_m) - c(q_1, \cdots, q_m, \theta).
\end{eqnarray*}
We can find $S(q_1, \cdots, q_m)>S(\hat{q}_1, \cdots, \hat{q}_m)$, which indicates that the declared malicious quality will negatively impact the winning probability. Thus, FMore is IC. 
\end{proof}

\section{Performance Evaluation}
In this section, we demonstrate the performance of FMore via both simulations and real-world experiments. A smart simulator is developed to comprehensively analyze the performance with a large number of edge nodes, while we also present the performance improvement with dynamic multi-dimensional resources in a realistic scenario.

\begin{figure*}[!th]
\centering
\setcounter{figure}{2}
\subfigure{
\begin{minipage}[t]{.48\linewidth}
\centering
\includegraphics[scale = 0.27]{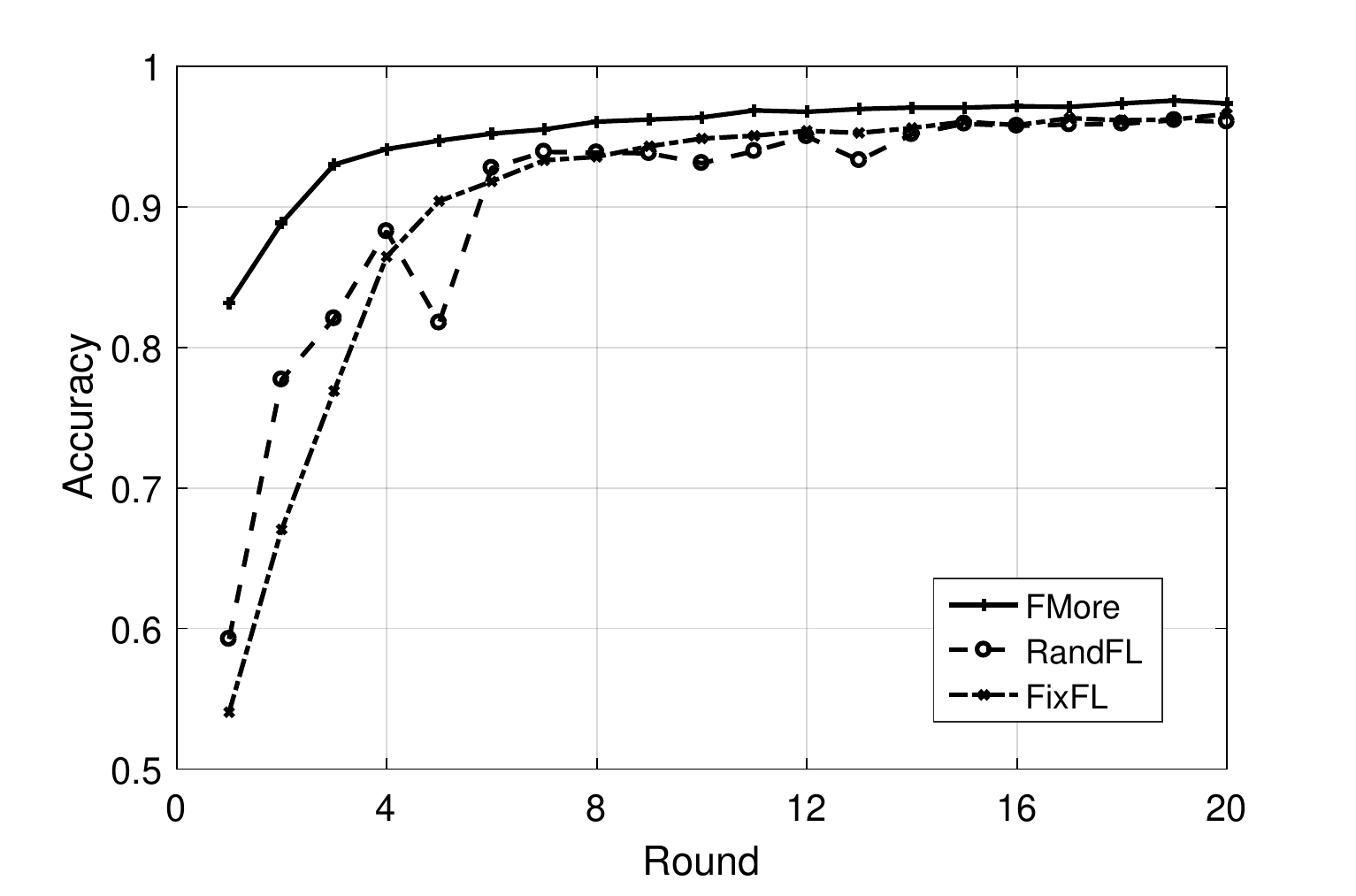}
\includegraphics[scale = 0.27]{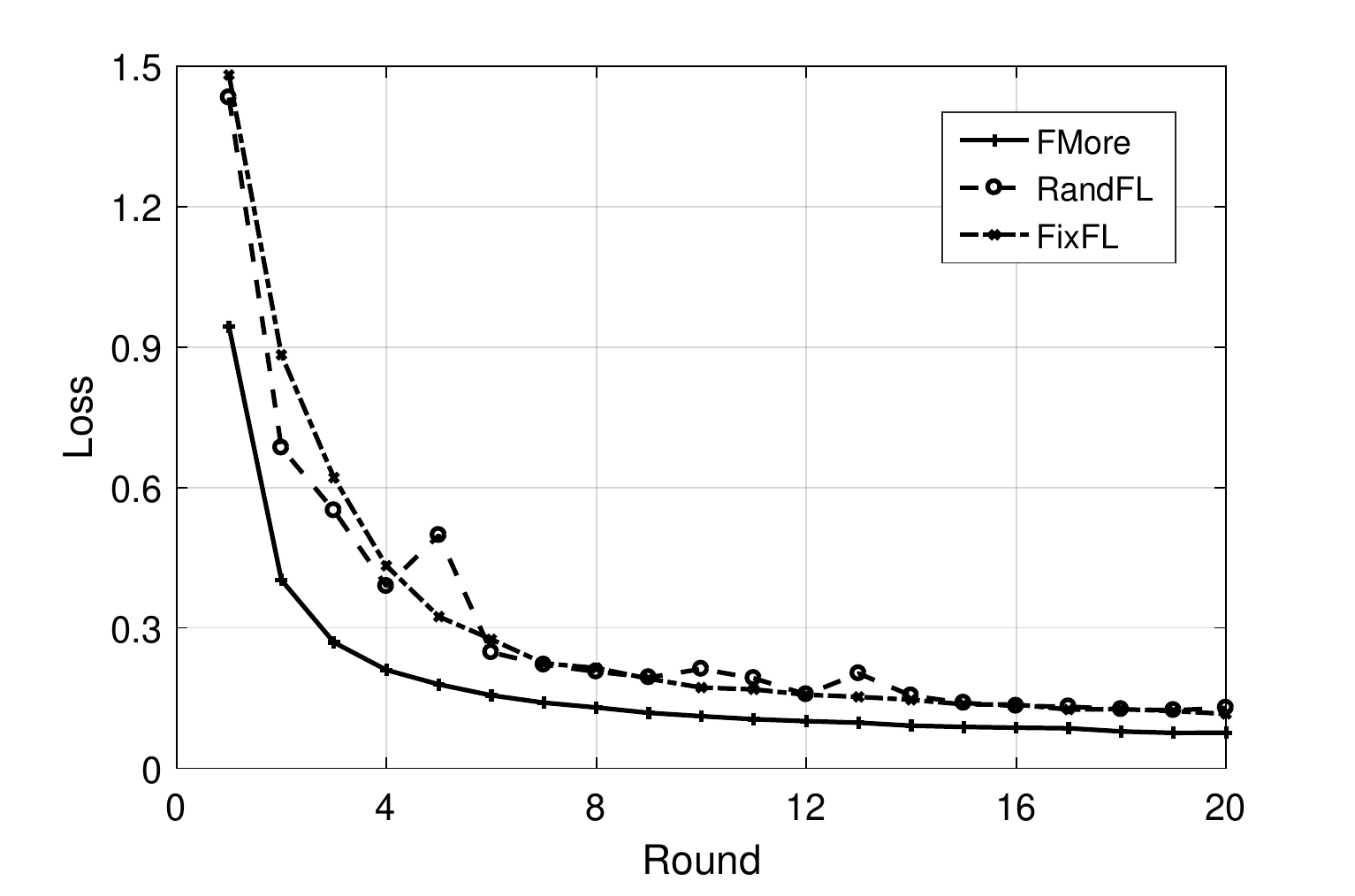}
\caption{The accuracy and loss for CNN with MNIST-O}
\label{CNNMNISTO}
\end{minipage}
}
\subfigure{
\begin{minipage}[t]{.48\linewidth}
\centering
\includegraphics[scale = 0.27]{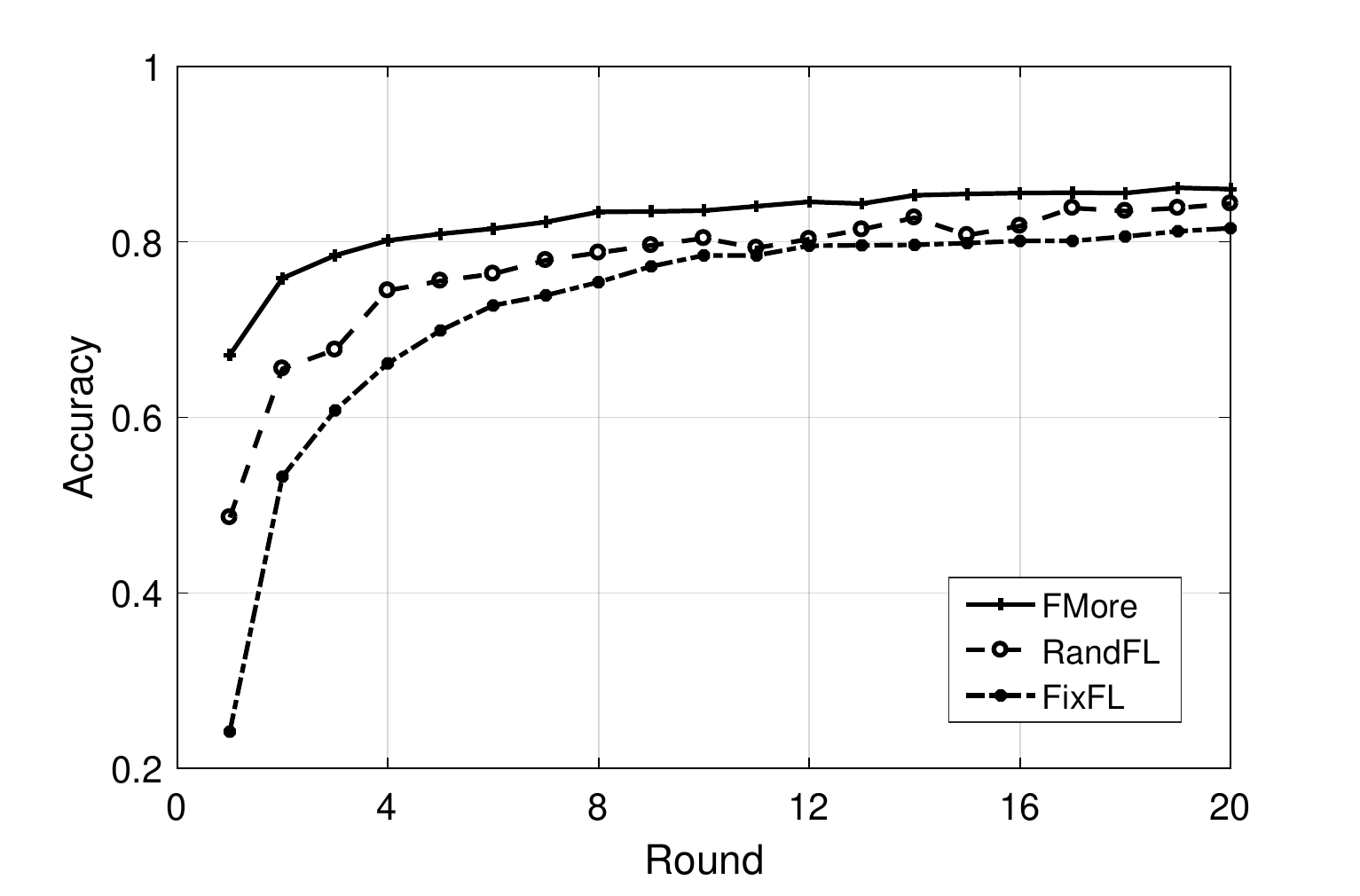}
\includegraphics[scale = 0.27]{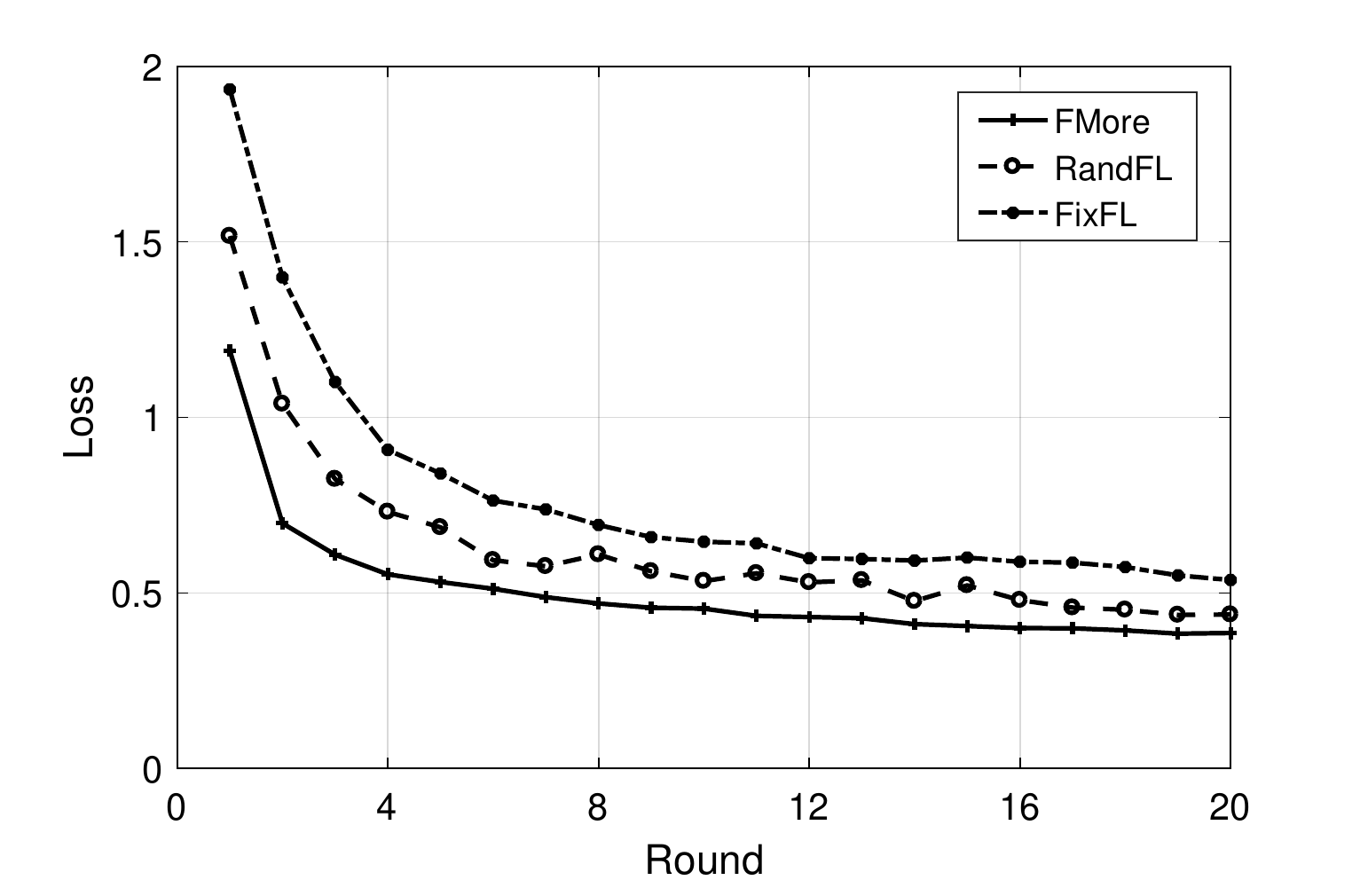}
\caption{The accuracy and loss for CNN with MNIST-F}
\label{CNNMNISTF}
\end{minipage}
}
\end{figure*}

\begin{figure*}[!ht]
\centering
\subfigure{
\begin{minipage}[t]{.48\linewidth}
\centering
\includegraphics[scale = 0.27]{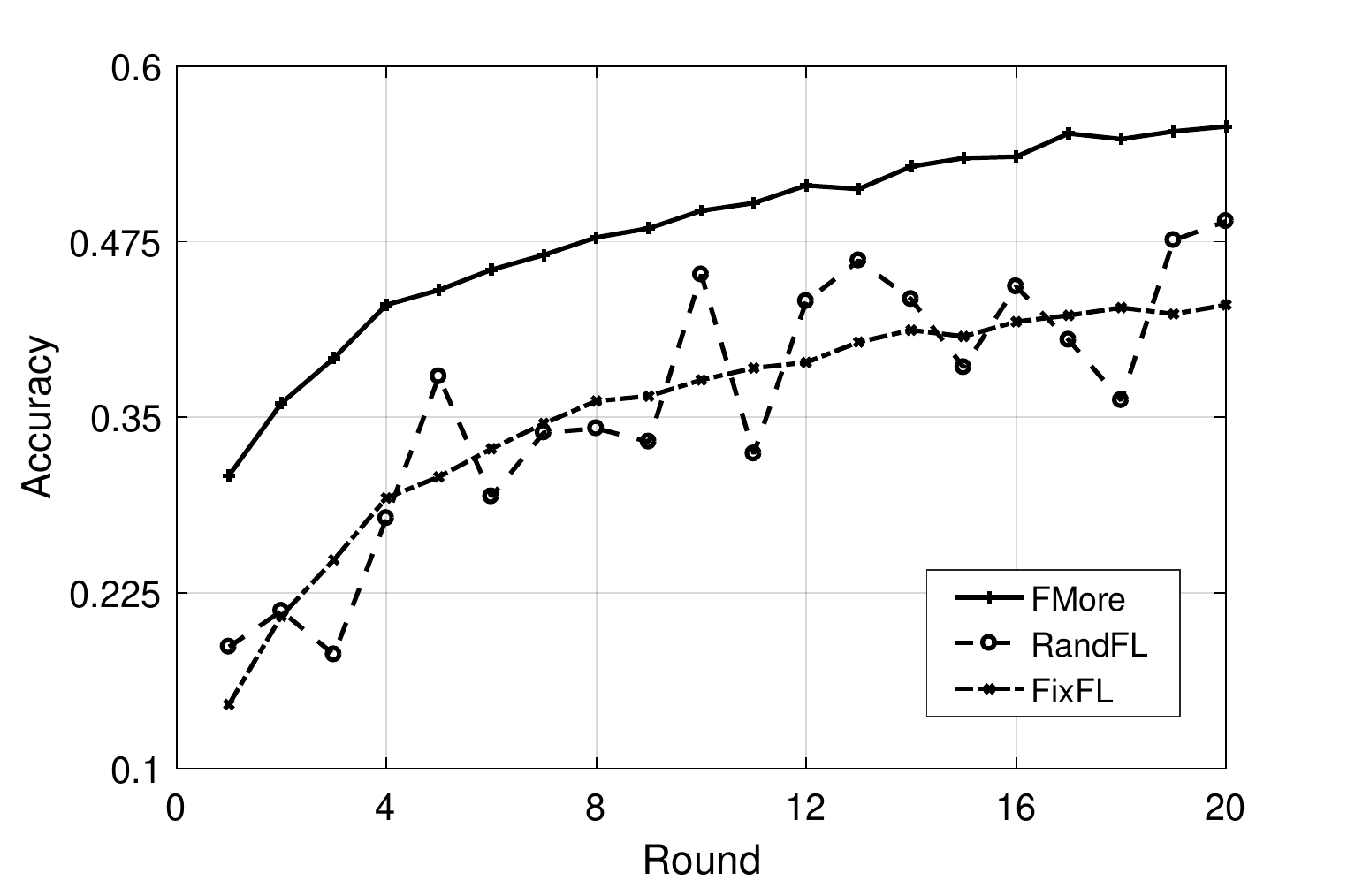}
\includegraphics[scale = 0.27]{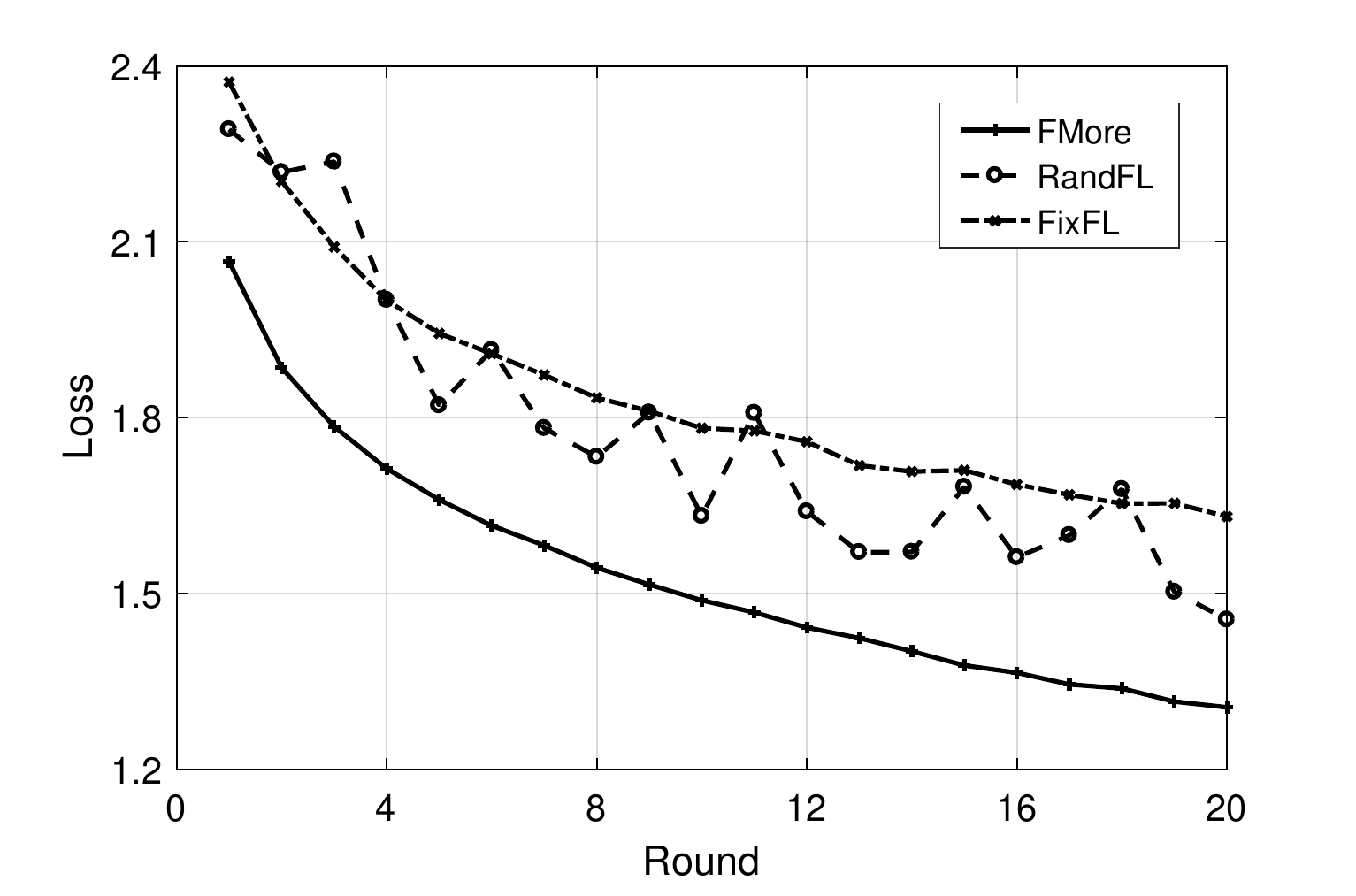}
\caption{The accuracy and loss for CNN with CIFAR-10}
\label{CNNCIF}
\end{minipage}
}
\subfigure{
\begin{minipage}[t]{.47\linewidth}
\centering
\includegraphics[scale = 0.27]{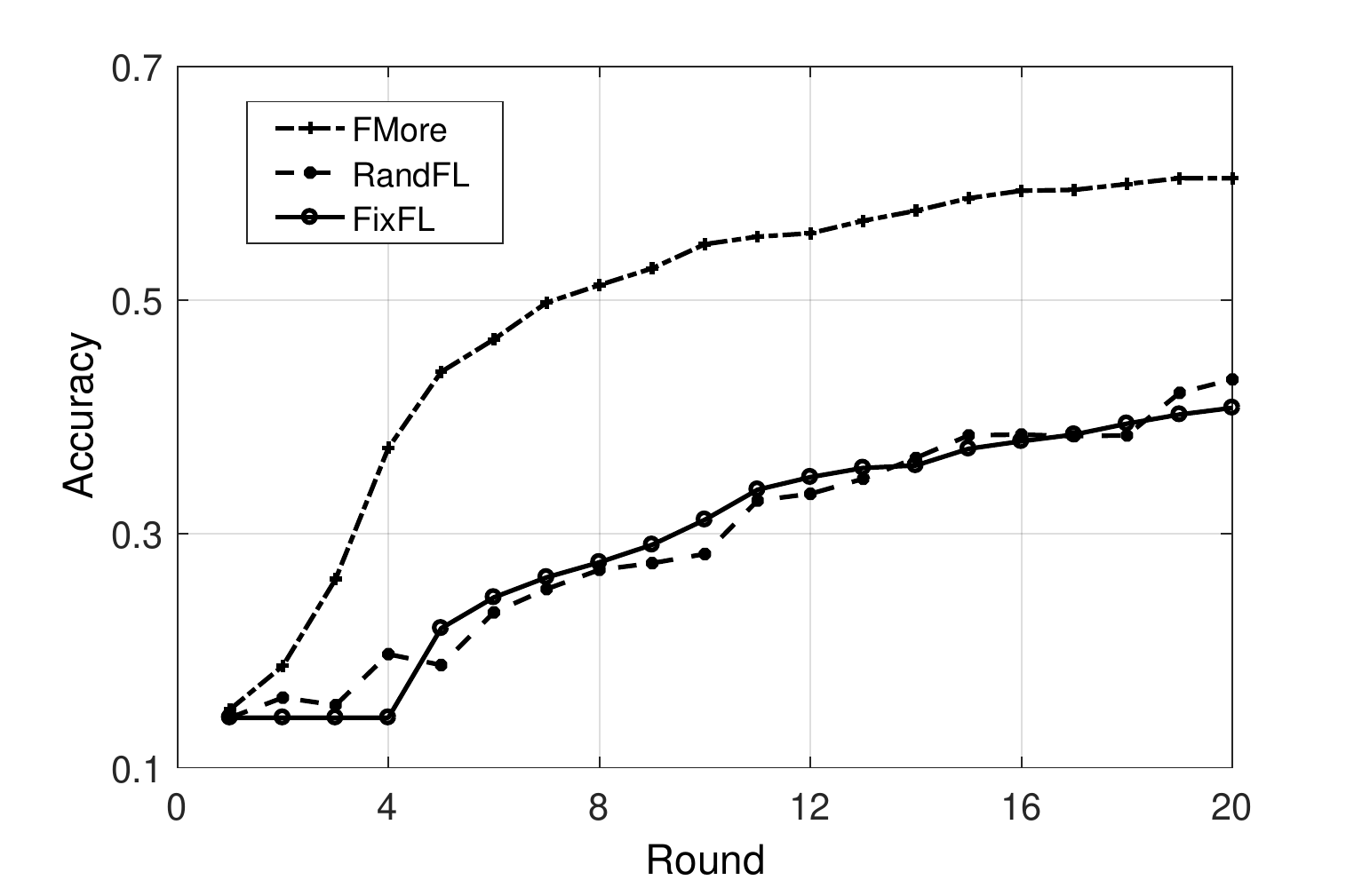}
\includegraphics[scale = 0.27]{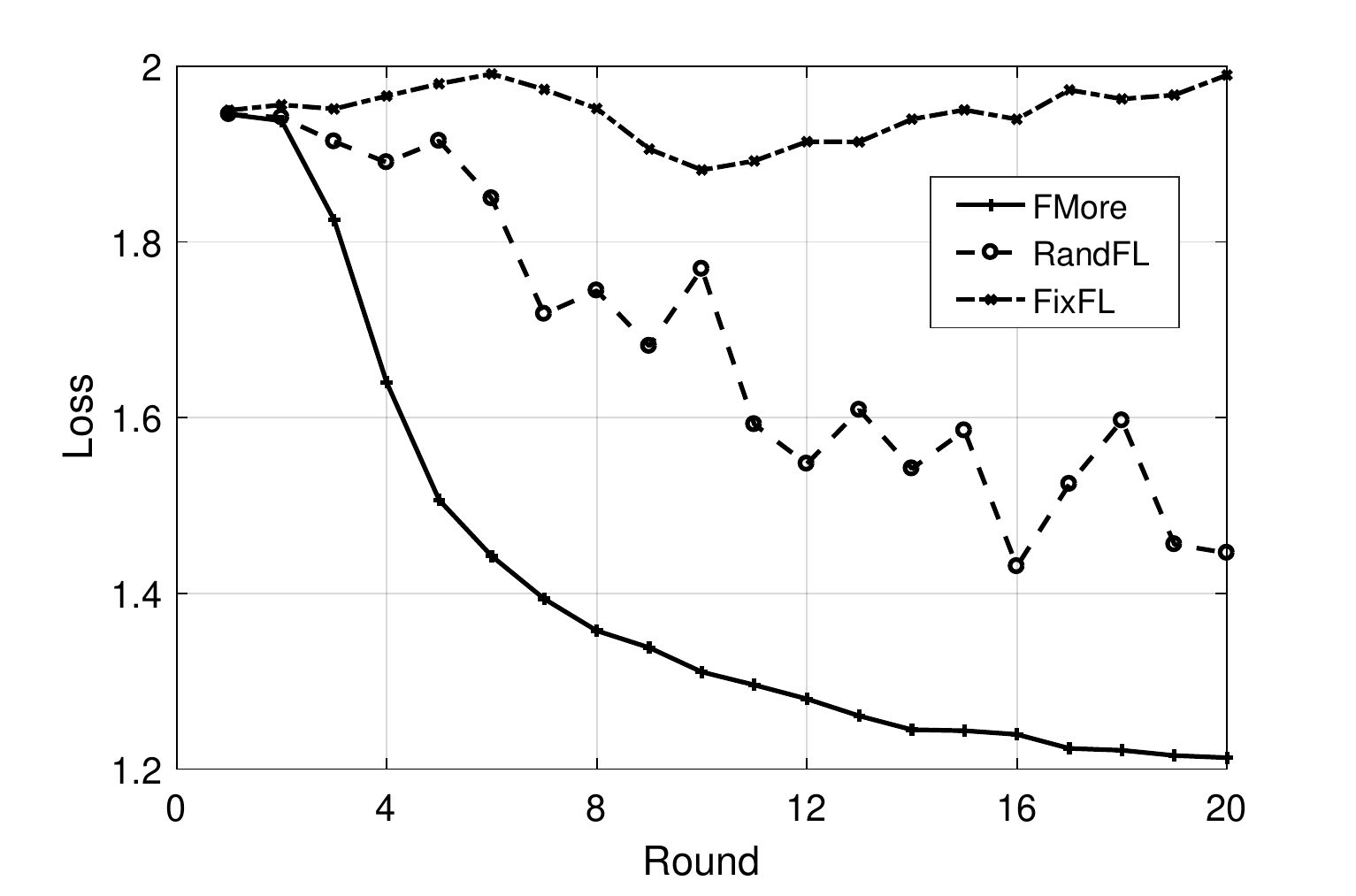}
\caption{The accuracy and loss for LSTM with HPNews}
\label{LSTMN}
\end{minipage}
}
\end{figure*}

\subsection{Setup}
We design a smart simulator based on Tensorflow to study the performance of FMore. In this simulator, we consider four classic datasets, i.e., MNIST (referred to as MNIST-O), Fashion MNIST (referred to as MNIST-F), CIFAR-10, and news category dataset (referred to as HPNews). The former three are collections of pictures. For instance, CIFAR-10 contains 60,000 color images of 10 different types of objects. The last news category dataset HPNews is a collection of 200,000 news headlines of HuffPost from 2012 to 2018. The underlying models include two CNNs (one\footnote{The CNN for MNIST has 8 layers with the following structure: $3\times3\times32$ Convolutional $\to$ $3\times3\times64$ Convolutional $\to$ $2\times2$ MaxPool $\to$ Dropout $\to$ Flatten $\to$ $1\times128$ Full connected $\to$ Dropout $\to$ $128\times10$ Fully connected $\to$ Softmax. This configuration is similar to the model in \cite{Google17:McMahan}.} for MNIST-O and MNIST-F, the other\footnote{The CNN for CIFAR-10 has 11 layers with the following structure: $3\times3\times32$ Convolutional $\to$ Dropout $\to$ $2\times2$ MaxPool $\to$ $3\times3\times64$ Convolutional $\to$ Dropout $\to$ $2\times2$ MaxPool $\to$ Flatten $\to$ Dropout $\to$ $1\times1024$ Full connected $\to$ Dropout $\to$ $1024\times10$ Fully connected $\to$ Softmax.} for CIFAR-10) and LSTM. Similar to \cite{Google17:McMahan}, non-IID data distribution of sample data is studied across different edge nodes. FMore, the classic federated learning (referred to as RandFL), and federated learning with fixed node selection (referred to as FixedFL) are all implemented here.

Our simulator consists of one aggregator and 100 participators ($N=100$). In each round of training, $K=20$ winners are selected to join in the cooperative training. The resources considered in the simulator are two-dimensional, i.e., data size $q_1$ and data category $q_2$. The participator computes the Nash equilibrium strategy via the Euler method. For the aggregator, the score function is $S(q_1, q_2, p) = s(q_1, q_2) - p = \alpha q_1 q_2 -p$, where $q_1$ is set to the data size, $q_2$ is the proportion of data category over the range of $(0,1]$, and the coefficient $\alpha$ is set to 25. The determination of winners is based on the first-score sealed auction. Ties are resolved by the flip of a coin. The default parameters are adopted throughout the entire simulations unless explicitly specified. 

We also implement FMore in a realistic HPC Cluster with one aggregator and 31 nodes. The specifications of these 32 nodes include Intel Core i7 CPU, 8GB DDR, 1T HDD+256G SSD, 1Gbps Ethernet, and Linux Ubuntu 18.04 OS. All these nodes are connected by a switch. The resources considered here include computing power, bandwidth, and data size. The scoring function is  $S(q_1, q_2, q_3, p) = \alpha_1 q_1 + \alpha_2 q_2 + \alpha_3 q_3  -p$, where the coefficients $q_1, q_2$, and $q_3$ are separately set as 0.4, 0.3, 0.3 for computing power, bandwidth and data size. The computing power is tuned by the number of CPU cores in the experiments. The data size is allocated over the range of $[2000, 10000]$ for the accuracy test. Nodes randomly choose different quantities of resources in each round of training. All the results are the average of five experiments for both simulations and real-world experiments in this section.

\subsection{Simulations}
The goals of FMore include motivating more nodes with high quality and low cost to participate in cooperative federated learning and improving the performance of global model. In essence, performance improvement is our final goal for federated learning. Here, we mainly discuss the results for the performance improvement of FMore from a variety of aspects.

(1) \textbf{Model Accuracy and Loss:} From Fig. \ref{CNNMNISTO} to Fig. \ref{LSTMN}, we can see that the model accuracy of FMore is larger than those of RandFL and FixFL after 20 rounds of training. When the underlying model is complicated or the training task is challenging, the accuracy gap between FMore and the other two is large, since training CIFAR-10 and HPnews needs more data with high quality and low cost. At the 20th round of learning in LSTM, the accuracy of FMore is 60.4\%, while FixFL is only 40.6\%. Another contribution is that FMore accelerates the training speed by 50\% for MNIST-O (accuracy 95\%),  42\% for MNIST-F (accuracy 84\%),  45\% for CIFAR-10 (accuracy 50\%), and 68\% for HPNews (accuracy 46\%), comparing with RandFL. The performance improvement is attributed to the selection of high-quality nodes, as shown in Fig. \ref{DistributionS}. Finally, the speedup of training especially benefits the aggregator, when the total payment of aggregator is constrained, or nodes would hesitate to participate in the training for a long time. 

(2) \textbf{The Impacts of Parameter $N$}: The increasing number of $N$ will improve data diversity and offer more opportunities for the aggregator to select nodes with high-quality resources and low cost, which in turn improves the accuracy and training speed. In Fig. \ref{ParaN}, the number of training round is reduced by 28\% to get the accuracy of 84\%, comparing $N=50$ with $N=100$. In each round, the accuracy with $N=100$ is larger than that with $N=50$. When $N$ is large enough and 10\% nodes are selected from all the nodes, the improvement of model accuracy is constrained. For $N=200$, data diversity is already satisfied. Moreover, the increase of $N$ reduces the payment $p$ for each node, which also benefits the aggregator.

(3) \textbf{The Impacts of Parameter $K$}: The large parameter $K$ reduces scores of winners since each node has more opportunities to be chosen. From Fig. \ref{ParaK}(b), we can find that the payment is increased as well, which conforms to Theorem 3. On the other hand, the large $K$ will feed the model with more data, which might benefit the model accuracy. In Fig. \ref{ParaK}(a), to get the accuracy of 86\%, 20 rounds of training are required for $K=5$, while 15 rounds are enough for $K=25$. It can be seen that the large $K$ speeds up the training process. When $K$ is too large, the margin profit of training speed is limited. The training results for $K=30$ and $K=35$ are similar in our simulations.

(4) \textbf{The Impacts of Parameter $\psi$}:  We also use the parameter $\psi$ to increase data diversity in some extreme scenarios. In Fig. \ref{PPsi}(b), we can find that the winner scores with $\psi = 0.2$ are more scattered than that with $\psi = 0.9$. When $\psi = 0.8$, almost 66.6\% nodes selected by $\psi$-FMore are among top 30 scores. When $\psi = 0.2$, $\psi$-FMore approaches to RandFL. Moreover,  $\psi$-FMore performs better than FMore in small-data-size scenarios which require more data diversity for federated learning.  Note that the increase of data diversity prevents the overfitting problem but sacrifices the speed of learning, shown in Fig. \ref{PPsi}(a). When $\psi = 0.3$, the accuracy only achieves 85\%, which can be achieved at the 11th round with $\psi = 0.9$.

\begin{figure}[!ht]
\centering
\vspace{0.05cm}
\setcounter{figure}{7}
\subfigure[CNN with CIFAR-10]{
\begin{minipage}[t]{0.46\linewidth}
\includegraphics[scale = 0.27]{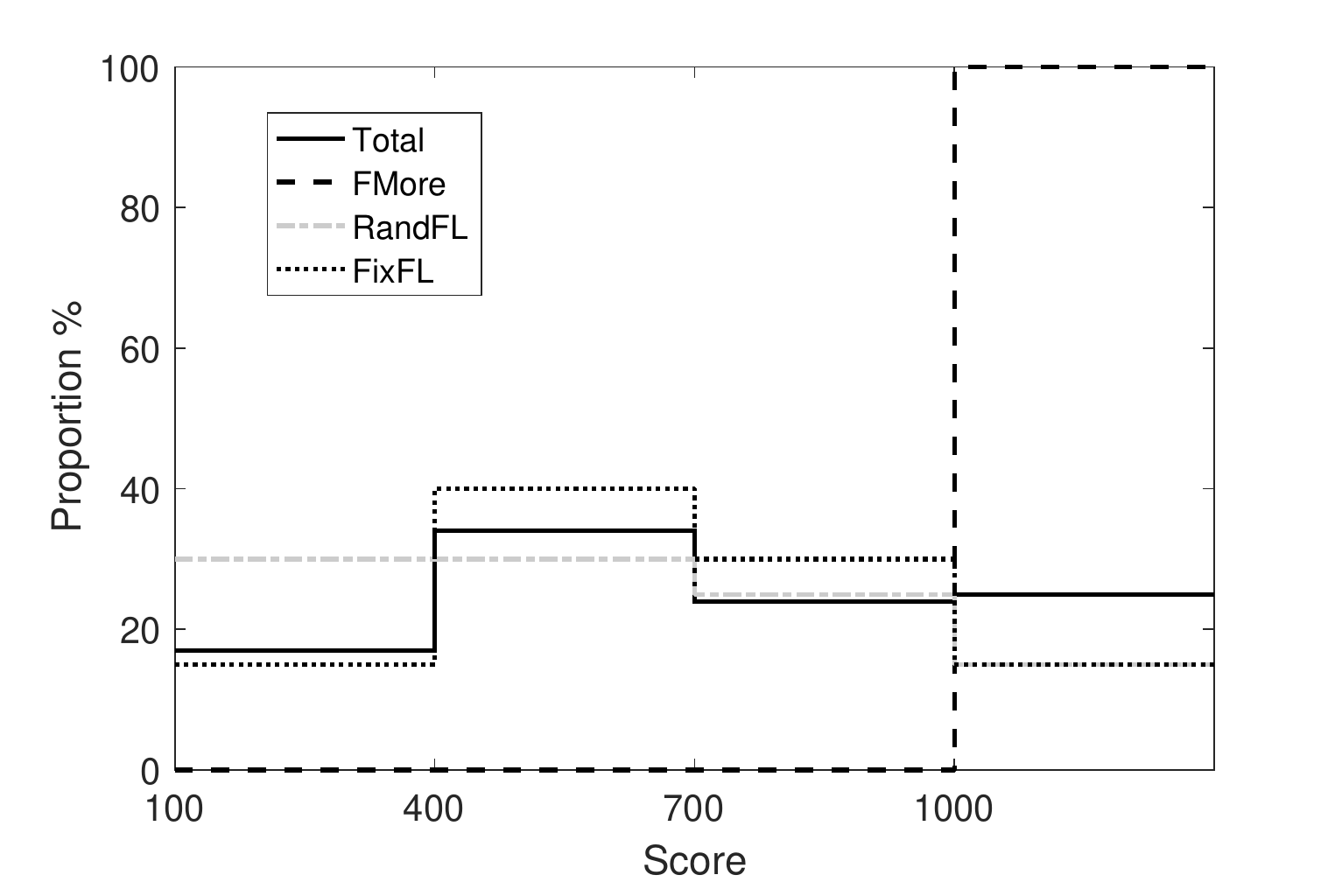}
\end{minipage}%
}%
\subfigure[LSTM with HPNews]{
\begin{minipage}[t]{0.52\linewidth}
\includegraphics[scale = 0.27]{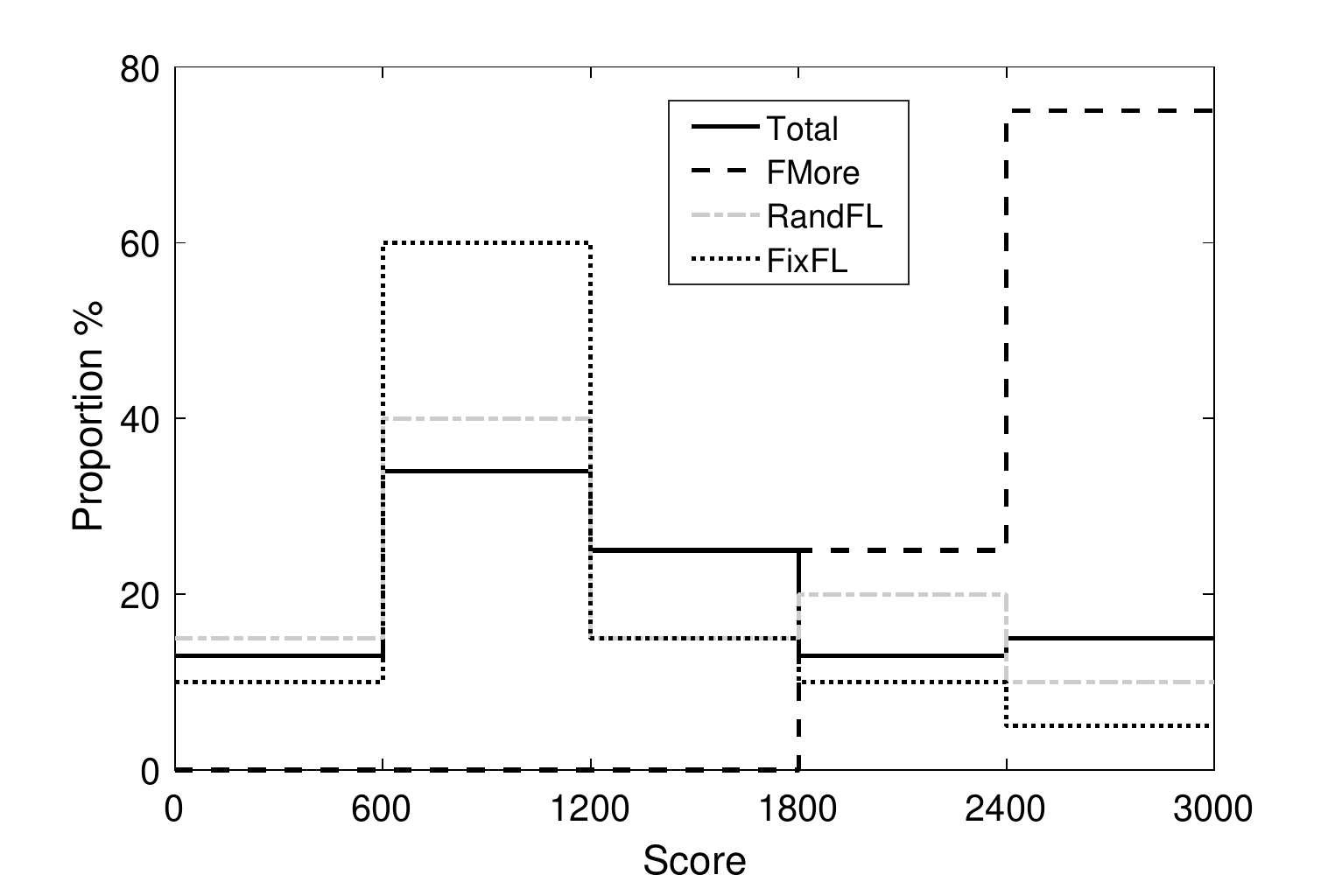}
\end{minipage}%
}%
\centering
\caption{The distribution of score}
\label{DistributionS}
\end{figure}

\begin{figure}[!ht]
\centering
\setcounter{figure}{9}
\subfigure[Rounds vs the accuracy]{
\begin{minipage}[t]{0.46\linewidth}
\includegraphics[scale = 0.27]{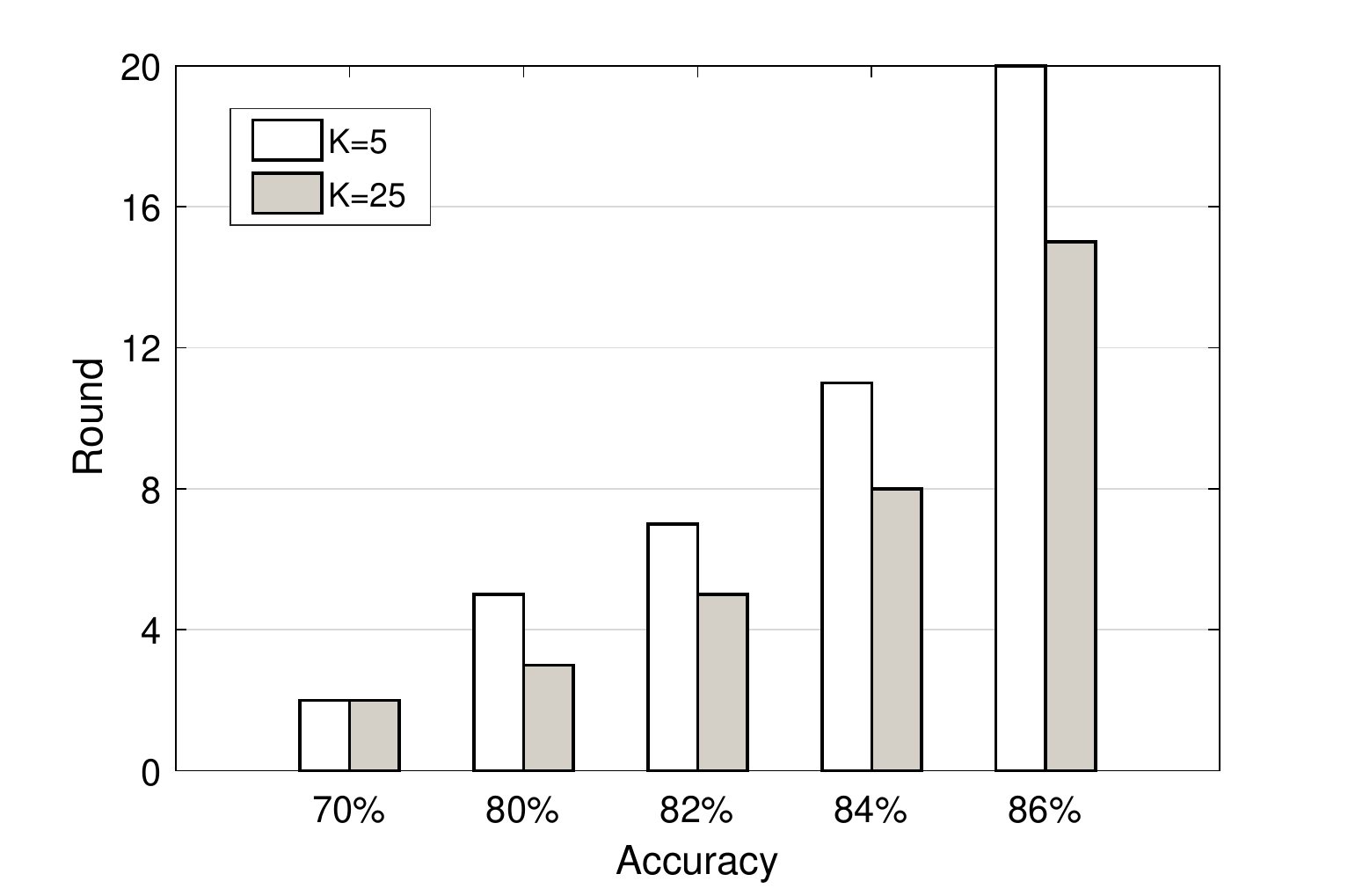}
\end{minipage}%
}%
\subfigure[Payment $p$ and score with $K$ ]{
\begin{minipage}[t]{0.52\linewidth}
\includegraphics[scale = 0.27]{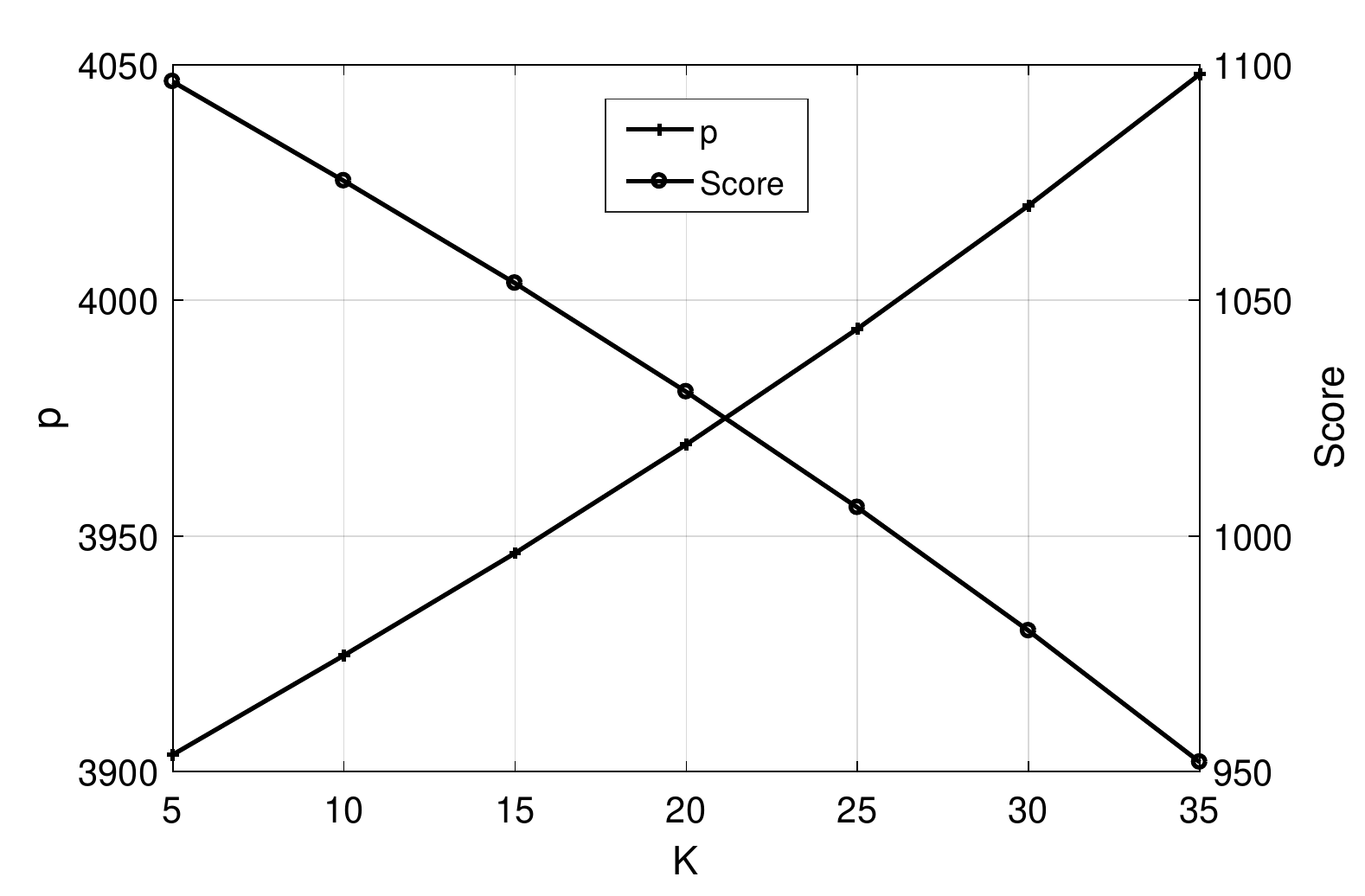}
\end{minipage}%
}%
\centering
\caption{The training speed and payment with parameter $K$}
\label{ParaK}
\end{figure}

\subsection{Real-world Experiments}
In the realistic deployment, the performance improvement of FMore includes two aspects, i.e., the accuracy improvement and the reduction of training time. From Fig. \ref{RelFig}, we can find that the model accuracy is 59.9\% for CIFAR-10 after the 20th round of training in FMore. FMore increases the accuracy by 44.9\%, comparing with RandFL. Similar accuracy improvement is shown in the LSTM model as well. In addition, there exist some accuracy jitters in RandFL. For the reduction of training time, FMore outperforms RandFL as well, as shown in Fig. \ref{RelFig2}. The total training time of 20 rounds is 1119.3s for CIFAR-10 in FMore, which reduces the training time by 38.4\%. To achieve the accuracy of 50\% for CIFAR-10, RandFL needs almost 17 rounds (1552.7s), while FMore only requires 8 rounds (427.7s). The advantages of FMore become increasingly prominent when we collaboratively training challenging AI tasks.  

\vspace{0.1cm}
\section{Related Work}
\subsection{Mobile Edge Computing}
MEC has drawn increasing attention in recent years \cite{IotJ:Abbas}. Most of the studies focus on service placement \cite{Infocom19:Poularakis}, \cite{Infocom19:Ou}, \cite{Infocom19:Gao}, task scheduling \cite{Infocom18:Xu}, deployment issues \cite{MECOMM18:Syamkumar}, \cite{TVT:Min}, etc. Among these studies, Wang firstly considered the performance issue of federated learning in MEC systems, and proposed an efficient control algorithm that trades off local updates and global aggregation to minimize the loss function with the constraint of resources \cite{Infocom18:Wang}.  Another interesting work is given in \cite{Networks19:Wang}, where both deep reinforcement learning and federated learning are employed to optimize edge computing, caching and communication in MEC. 

\begin{figure}[!ht]
\centering
\setcounter{figure}{8}
\subfigure[Rounds vs the accuracy]{
\begin{minipage}[t]{0.46\linewidth}
\includegraphics[scale = 0.27]{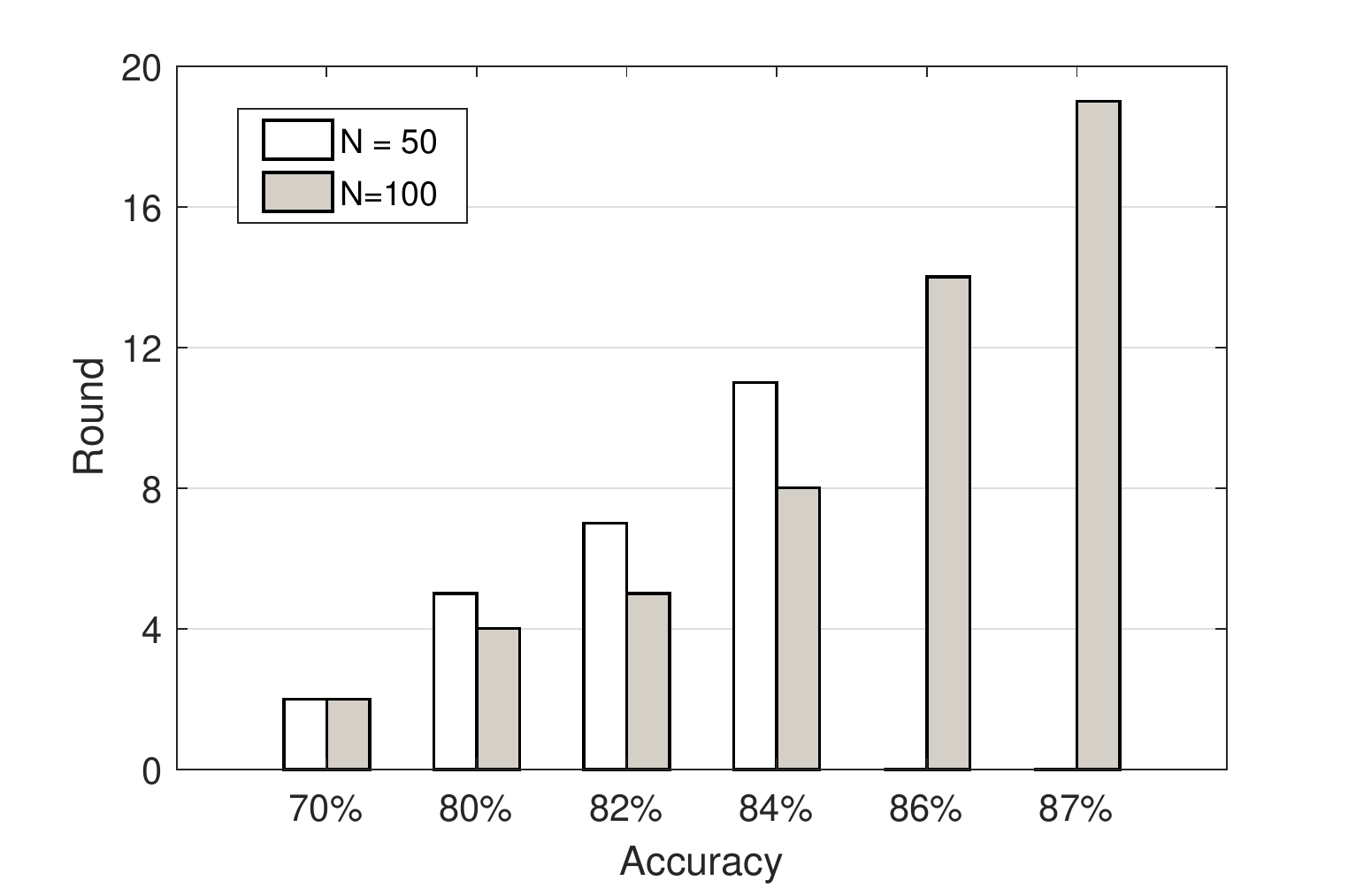}
\end{minipage}%
}%
\subfigure[Payment $p$ and score with $N$ ]{
\begin{minipage}[t]{0.52\linewidth}
\includegraphics[scale = 0.27]{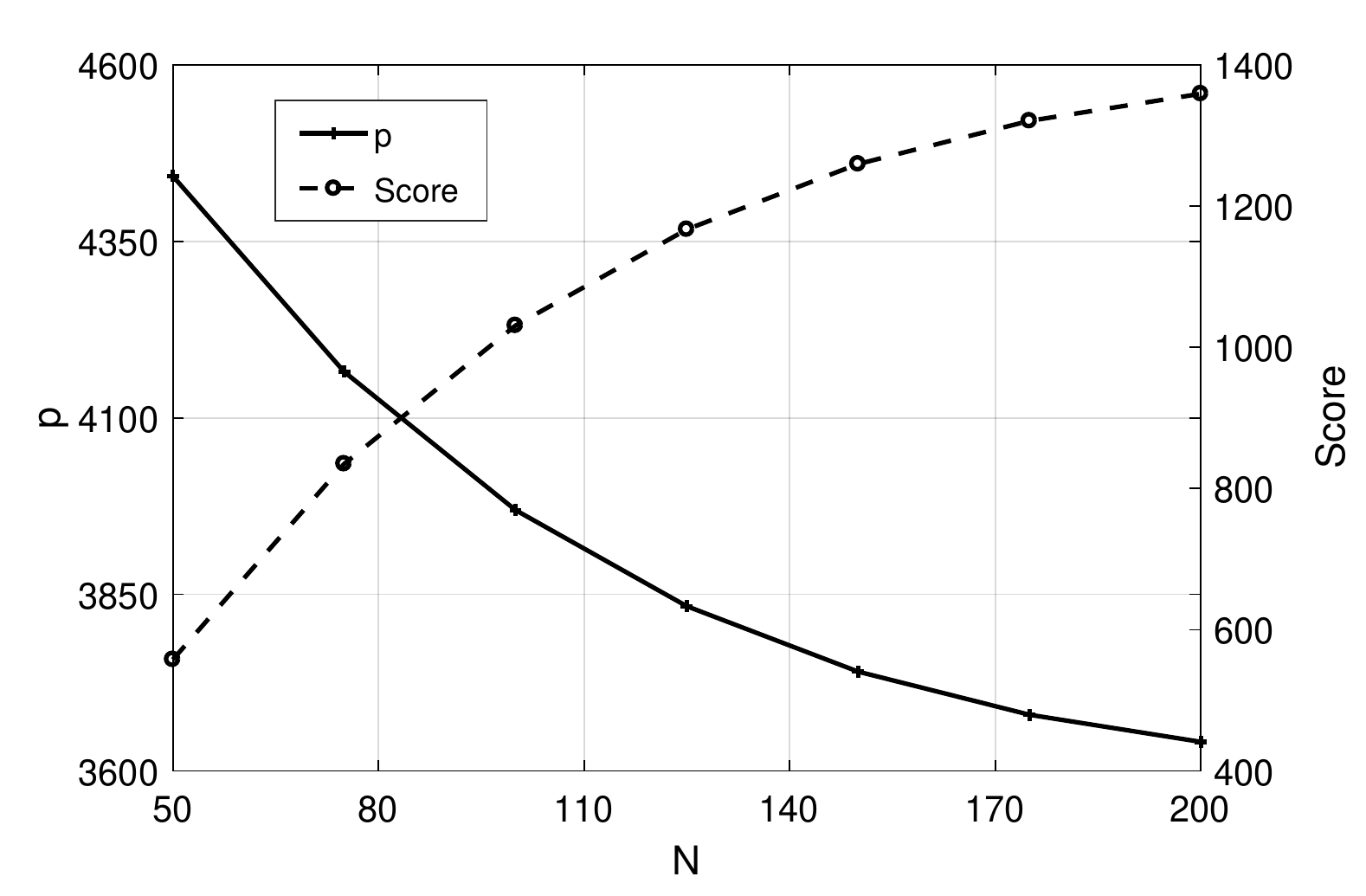}
\end{minipage}%
}%
\centering
\caption{The training speed and payment with parameter $N$}
\label{ParaN}
\end{figure}

\begin{figure}[!ht]
\centering
\setcounter{figure}{10}
\subfigure[The training speed with $p$]{
\begin{minipage}[t]{0.46\linewidth}
\includegraphics[scale = 0.27]{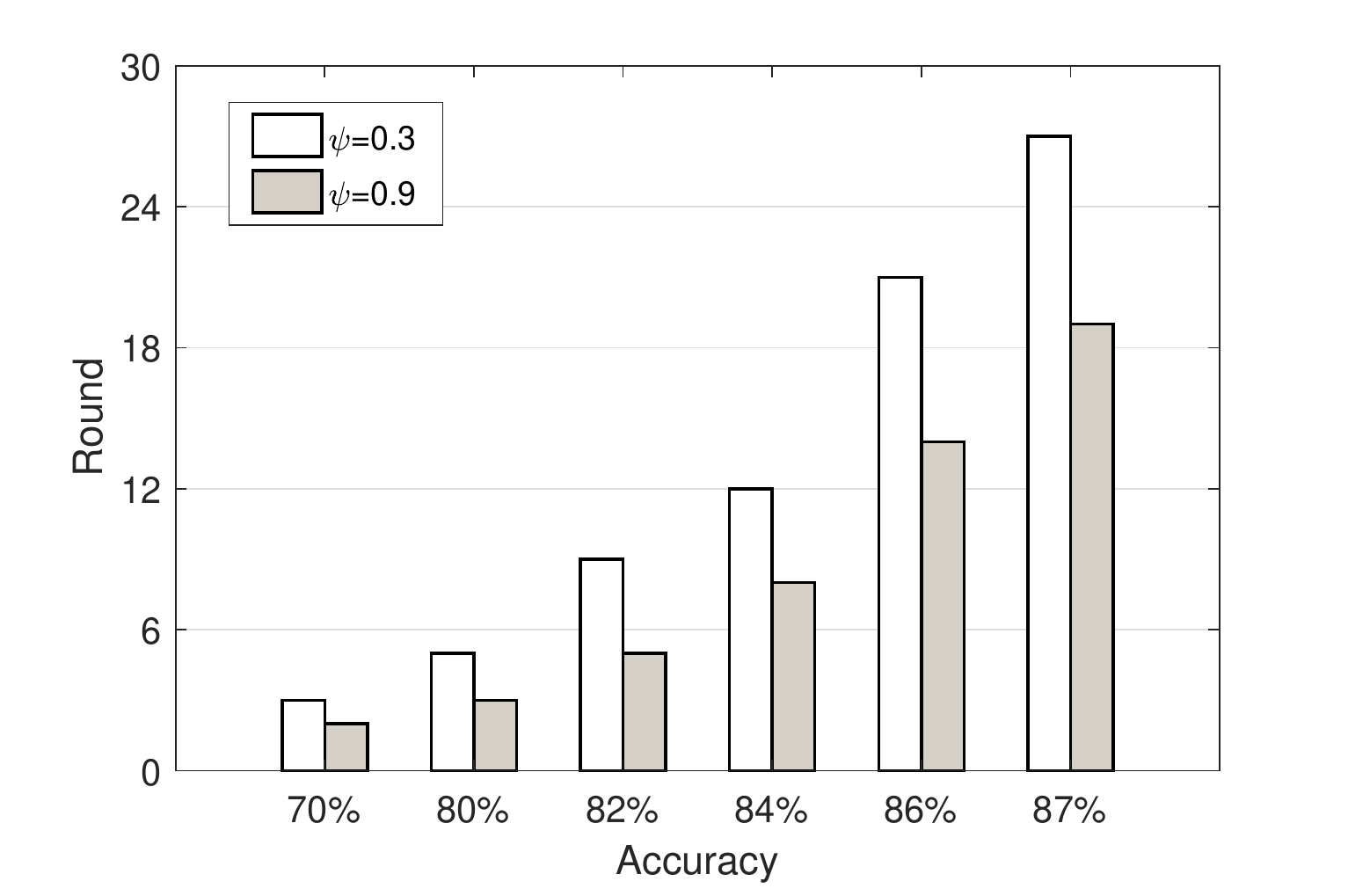}
\end{minipage}%
}%
\subfigure[The proportion of selected node]{
\begin{minipage}[t]{0.52\linewidth}
\includegraphics[scale = 0.27]{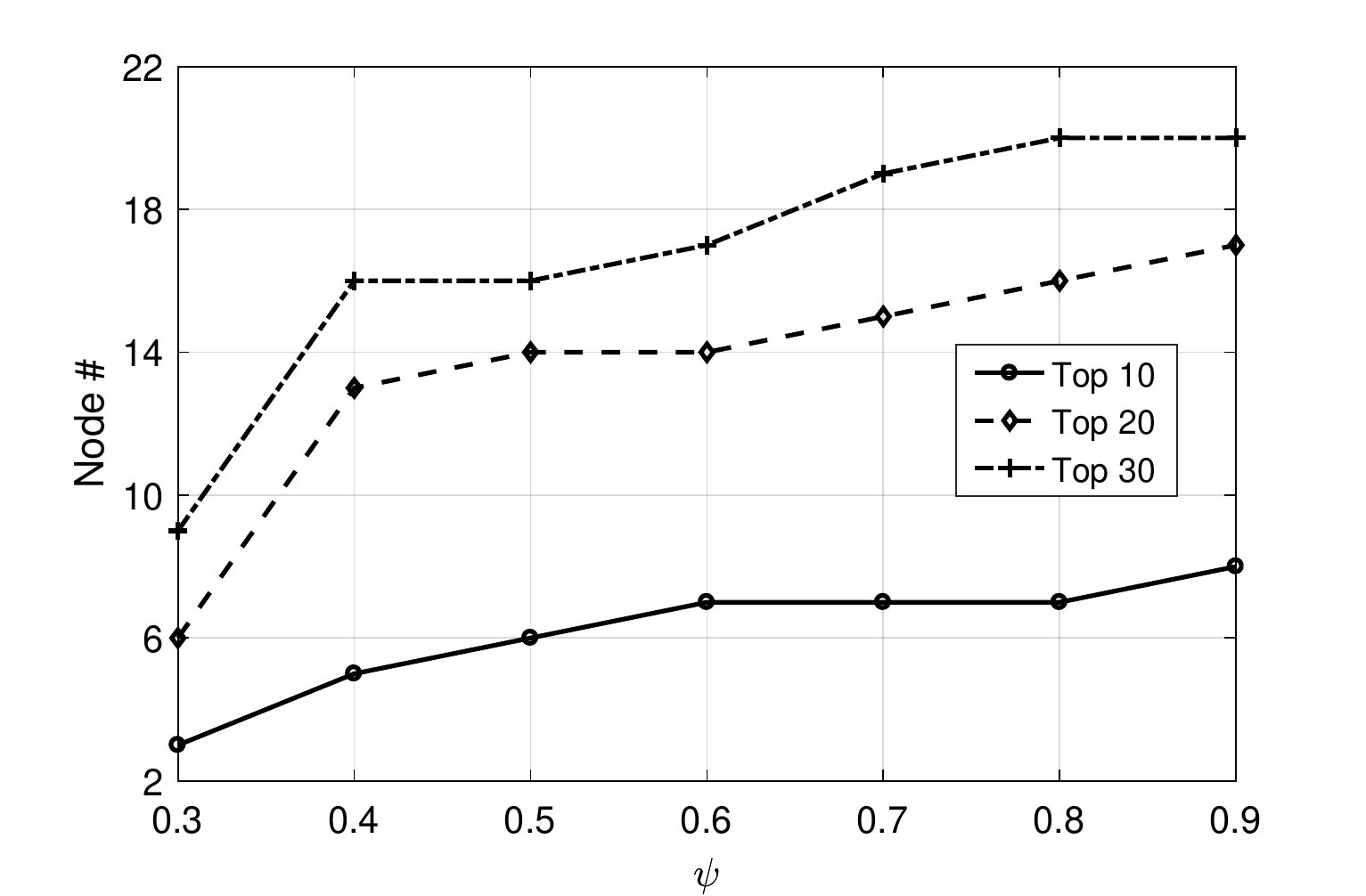}
\end{minipage}%
}%
\caption{The performance impacts of parameter $\psi$}
\label{PPsi}
\end{figure}

\subsection{Federated Learning}
Federated learning, firstly proposed by McMahan in \cite{Google17:McMahan}, has become a fascinating topic in the machine learning community \cite{Others19:Li}.  It is designed for privacy-concerning scenarios where local nodes would not like to upload and share their private data. Federated Learning is quite different from another impressive technique called distributed machine learning \cite{INFOCOM2019:Chu}. Distributed machine learning is adopted to deal with the massive data set and partition subsets of data to many nodes. Nowadays, plenty of studies focus on federated learning. Kairou et. al. summarized 438 papers and presented recent advances and open problems in the field of federated learning \cite{Others19:Kairouz}.

Many papers concentrate on the performance improvement of federated learning \cite{Infocom18:Bao}, \cite{Infocom19:Bao}. In \cite{Others19:Sattler}, Sattler proposed the sparse ternary compression scheme for non-IID data. Zhao also focused on the non-IID data and presented a method of sharing a small subset of data between all the edge nodes to improve the accuracy of federated learning \cite{Others18:Zhao}.  For Stochastic Gradient Descent (SGD), Wang presented and analyzed the cooperative SGD method, and provided convergence guarantees for the existing algorithms \cite{Others19:Wang}. In \cite{AAAI19:Yu}, Yu provided theoretical studies on the comparison of model averaging and mini-batch SGD. Nishio studied the node selection problem with resource constraints and provided a heuristic algorithm to find qualified nodes \cite{ICC19:Nishio}. This paper is a little similar to our work, but Nishio neglects the incentive mechanism, which is quite significant for MEC systems.

Both security and privacy are important concerns for federated learning \cite{Others18:Kim}. Bonawitz proposed a secure global aggregation algorithm that allows the server to compute without learning each user's contribution \cite{CCS17:Bonawitz}. Impressively, Wang explored the user-level privacy leakage against federated learning by attacks from malicious servers, and proposed a framework with GAN to discriminate category and client identity of input samples \cite{Infocom19:Wang}. For the local privacy, Bhowmick designed an optimal locally differentially private scheme for statistical learning problems \cite{Others19:Bhowmick}. In \cite{Applied18:Preuveneers}, Preuveneers considered attacks from local models with malicious training samples and provided a chained anomaly detection method for federated learning. In \cite{Others18:Bagdasaryan}, Bagdasaryan identified that participators can inject hidden backdoors into the global model and proposed a new model poisoning methodology with model replacement. 

All these studies assume that edge nodes voluntarily participate in federated learning, without requiring any returns, which does not hold in the realistic scenario of MEC. They neglect the incentive issue of federated learning, except for the work in \cite{Others19:Kang}. Kang utilized the contract theory to motivate nodes to participate in federated training in mobile networks \cite{Others19:Kang}. Kang and we almost simultaneously discover the significance of the incentive issue. However, there exist some main differences: (1) Kang considers the incentive problem in the monopoly market, where mobile terminals can only decide whether to accept the contracts or not. The efficiency of Kang's scheme is decided by the total number of contracts. In our scheme, edge nodes have more opportunities to submit any combination of resources and the expected payment, and the buyer ``aggregator'' can choose any node with qualified resources. (2) In Kang's scheme, the computational cost of calculating the optimal contracts is NP-hard, while edge nodes only need linear time to get the optimal strategy in FMore. (3) Node selection is not provided in \cite{Others19:Kang}, while we not only motivate more high-quality edge nodes to participate in the training but also select those suitable nodes with low costs.

\begin{figure}[!t]
\centering
\setcounter{figure}{11}
\subfigure{
\begin{minipage}[t]{0.46\linewidth}
\includegraphics[scale = 0.27]{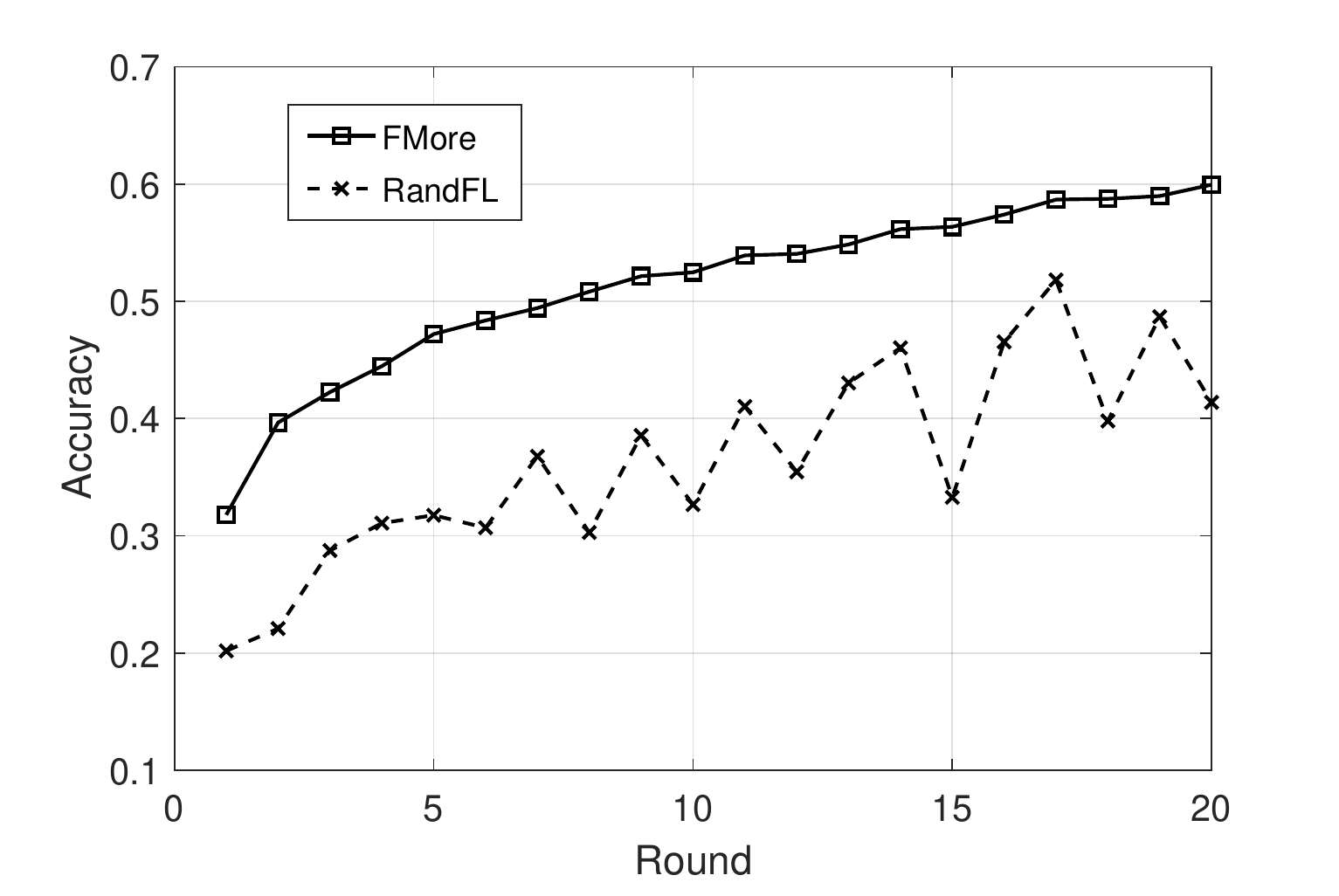}
\end{minipage}%
}%
\subfigure{
\begin{minipage}[t]{0.52\linewidth}
\includegraphics[scale = 0.27]{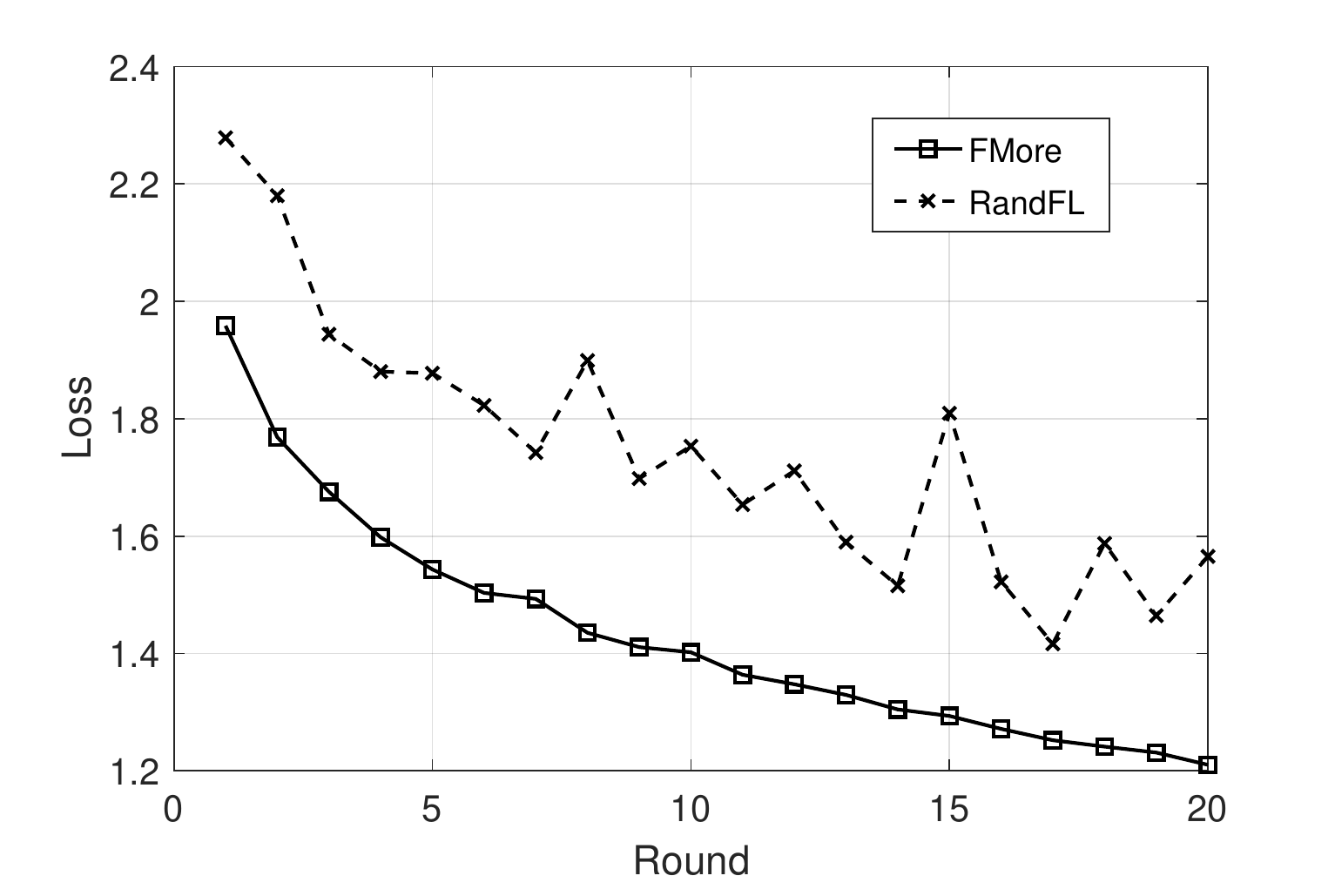}
\end{minipage}%
}%
\centering
\caption{The performance of CIFAR-10 in realistic deployment}
\label{RelFig}
\end{figure}

\subsection{Procurement Auction}
The incentive schemes with procurement auction are designed to address varieties of problems, such as the allocation of radio-frequency spectrum \cite{TMC18:Lin}, crowdsensing \cite{Ton16:Yang}, display advertising \cite{Sigkdd14:Zhang}, and client-assisted cloud storage systems \cite{TMC18:Jin}, \cite{INFOCOM2015:Zhao}. Unfortunately, none of them can be directly applied to federated learning in MEC, since they just considered the specific property in their problems. In MEC, resources provided by edge nodes are multi-dimensional and dynamic, and $K$ winners are selected in each game. In addition, the proposed scheme should be lightweight and able to improve the performance of federated learning with well-chosen nodes. Consequently, we need to design a novel incentive scheme for federated learning in MEC. 

\begin{figure}[!t]
\centering
\setcounter{figure}{12}
\subfigure{
\begin{minipage}[t]{0.46\linewidth}
\includegraphics[scale = 0.27]{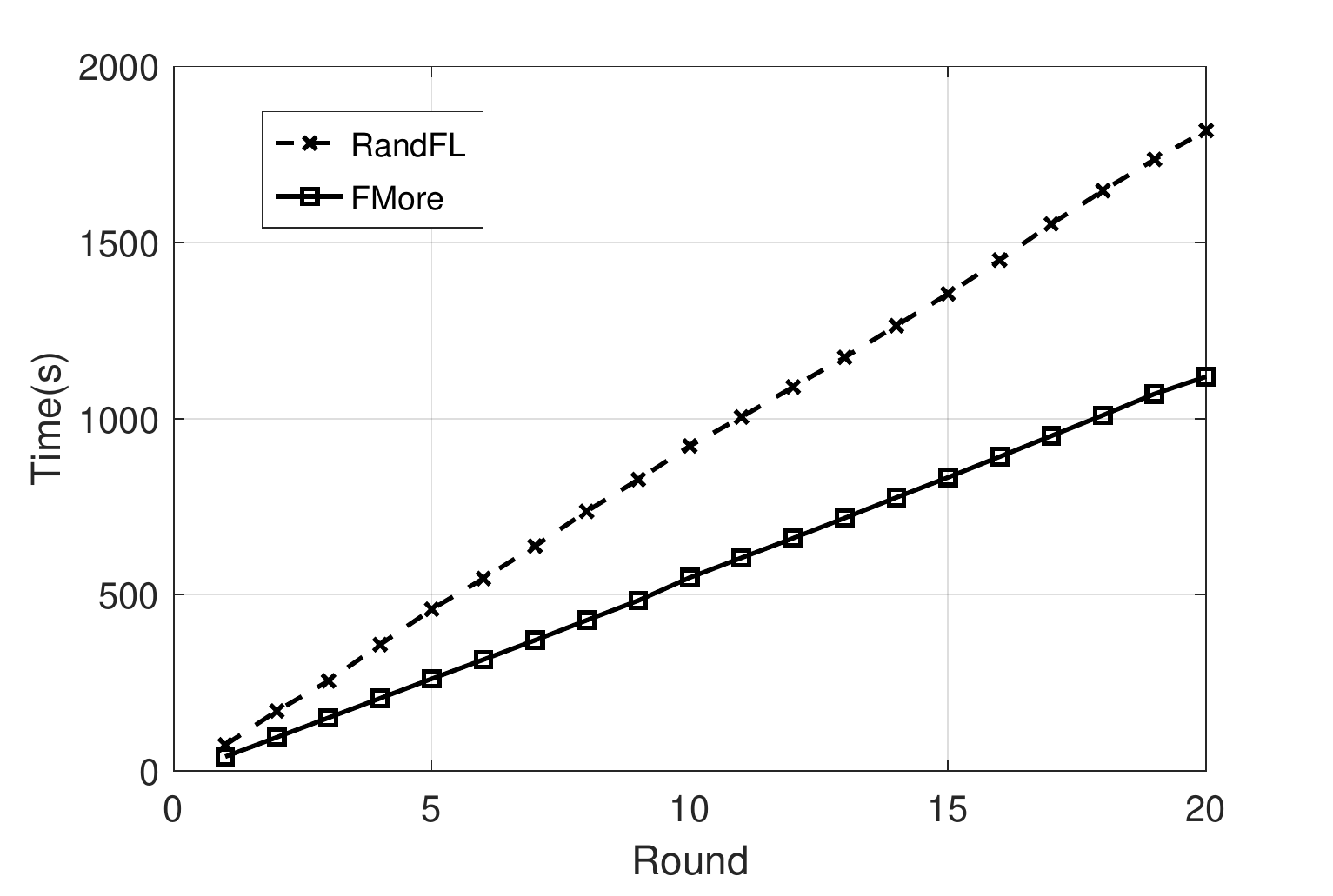}
\end{minipage}%
}%
\subfigure{
\begin{minipage}[t]{0.52\linewidth}
\includegraphics[scale = 0.27]{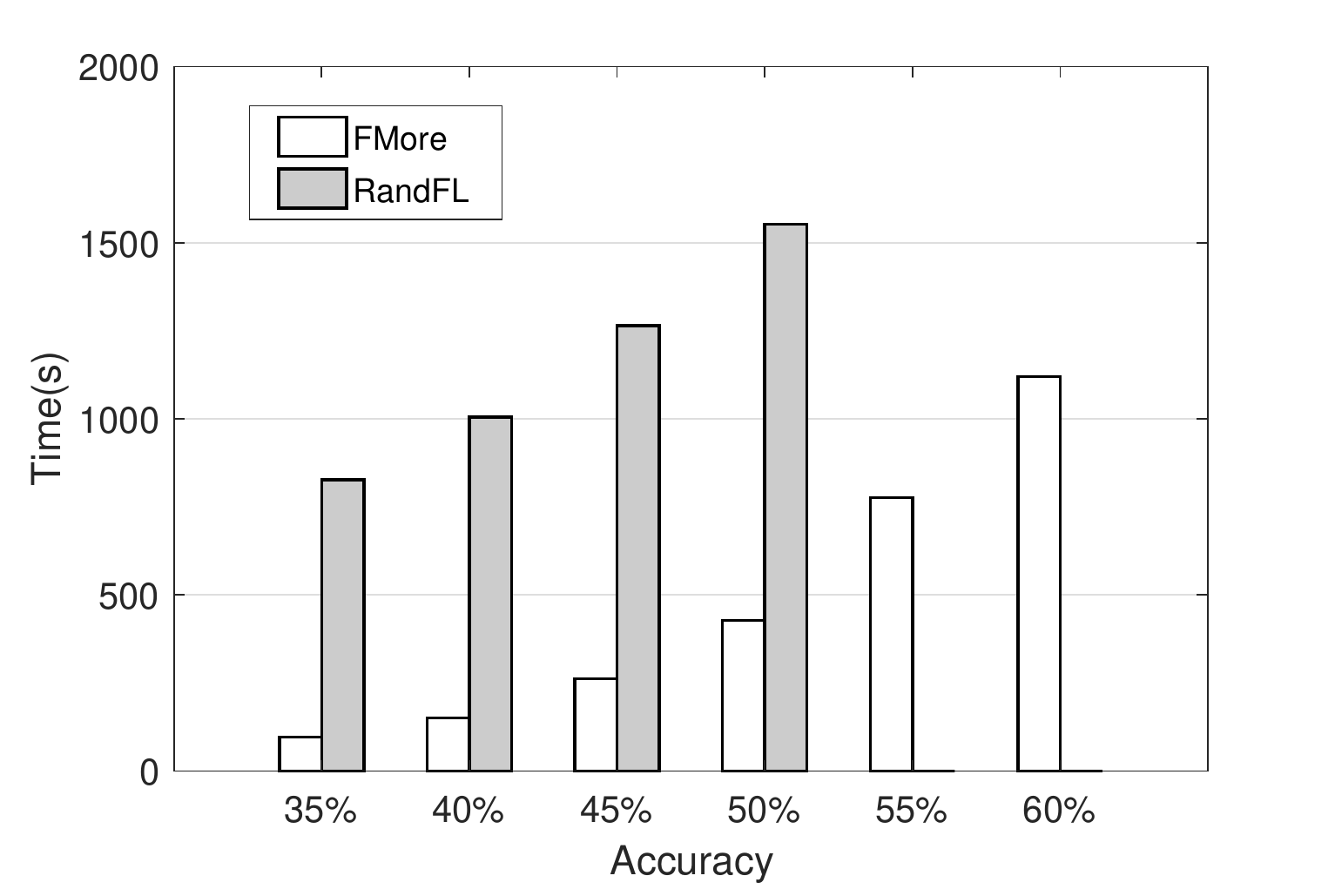}
\end{minipage}%
}%
\centering
\caption{The training speed for CIFAR-10 in realistic deployment}
\label{RelFig2}
\end{figure}

\section{Conclusion}  
In this paper, we have considered the incentive mechanism for federated learning in MEC and proposed a lightweight and efficient scheme FMore. FMore adopts the multi-dimensional procurement auction with $K$ winners. In FMore, edge nodes can obtain the Nash equilibrium strategy from our theoretical results in linear computation time. We also provide guidance to the aggregator to get the expected resources. We develop a simulator with two models and four datasets to demonstrate the advantages of FMore. Extensive simulations show that FMore can reduce the training rounds by almost 51.3\% and improve the accuracy by 28\% for the LSTM model. We also implement a real-world system with 32 nodes in a Linux HPC cluster, and find that the training time is reduced by 38.4\% while model accuracy is increased by 44.9\%. In this paper, the budget constraint of the aggregator is not considered, which is left for future work. In addition, whether the probability $\psi$ should be identical or distinct for each node remains to be studied.

\appendix
\subsection{The proof of Proposition 2}
\begin{proof}
Since all the nodes have the same private cost parameter $\theta$, we can find that their scores are identical, according to Che's Theorem 1 and Theorem 1. Since we must select $K$ nodes from $N$ nodes, and each node is being selected with the same probability, then it will be put in the set $\mathbb{W}$ with probability $\frac{K}{N}$, which is not related with probability $\psi$. 
\end{proof}
\subsection{The proof of Proposition 3}
\begin{proof}
Suppose that an edge node has a Nash equilibrium strategy $(q_1, \cdots, q_m, p)$, which means the expected utility is maximized by this bidding strategy. In the following, we will obtain the contradiction that there exists an alternative strategy $(q_{s1}, \cdots, q_{sm}, \hat{p})$, where $(q_{s1}, \cdots, q_{sm}) = arg \, max \, s(q_{s1}, \cdots, q_{sm}) - c(q_{s1}, \cdots, q_{sm},\theta)$ and $\hat{p} = p + s(q_{s1}, \cdots, q_{sm}) - s(q_1, \cdots, q_m)$, and it dominates the Nash equilibrium strategy $(q_1, \cdots, q_m, p)$. We first find that these two strategies have the same scores, i.e., $ S(q_{s1}, \cdots, q_{sm}, \hat{p}) = S(q_{1}, \cdots, q_{m}, p)$. This can be easily proved by substituting $\hat{p}$ with $ p + s(q_{s1}, \cdots, q_{sm}) - s(q_1, \cdots, q_m)$ in the score function. Then, we can have $\pi(q_{s1}, \cdots, q_{sm}, \hat{p}) \ge \pi(q_{1}, \cdots, q_{m}, p)$, and get the expected contradiction. This demonstrates that the qualities could be chosen with the optimization problem $arg \, max \, s(q_{s1}, \cdots, q_{sm}) - c(q_{s1}, \cdots, q_{sm},\theta)$. 
\end{proof}

\subsection{The proof of Proposition 4}
\begin{proof}
According to the expected utility theory, the aggregator needs to solve the following optimization problem with the cost constraint: 
\begin{eqnarray*}
&&max \ s(\cdot) = \prod_{i=1}^m q_i^{\alpha_i}\\
&&s.t. \sum_{i=1}^m{\tilde{\beta}_i} q_i \theta= c_0
\end{eqnarray*}
where $\tilde{\beta}_i$ is the estimation of coefficients for $q_i$ and $c_0$ is the budget. We use the Lagrange multiplier method to compute the above problem and have 
\begin{equation*}
L = s(q_1, \cdots, q_m) - \lambda (\theta( \sum_{i=1}^m{\tilde{\beta}_i} q_i) - c_0)
\end{equation*}
Then, we can get 
\begin{equation*}
\frac{\partial L}{\partial q_i } =  \frac{\alpha_i \prod_{i=1}^m q_i^{\alpha_i}}{q_i} -\theta \lambda {\tilde{\beta}_i} =0
\end{equation*}
Thus, we can easily obtain the conclusion 
\begin{equation*}
\frac{q_i^*}{q_j^*} = \frac{\alpha_i}{\alpha_j} \cdot \frac{\tilde{\beta}_j}{\tilde{\beta}_i},
\end{equation*}
where $q_i^*$ and $q_j^*$ are the optimal choices for the aggregator. 
\end{proof}

\end{document}